\documentclass[11pt]{amsart}

\usepackage{verbatim, amssymb,hyperref}

\usepackage{color}
\usepackage{bbm}
\usepackage{graphicx}

\newcommand{\einschraenkung}{|}

\newcommand{\cF}{\mathcal F}

\newcommand{\cC}{\mathcal C}

\newcommand{\cG}{\mathcal G}

\newcommand{\cO}{\mathcal O}

\newcommand{\Var}{\mathrm{Var}}

\newcommand{\grad}{\mathrm{grad}\,}

\newcommand{\IW}{\mathbb W}
\newcommand{\cW}{\mathcal W}

\newcommand{\rd}{d}

\newcommand{\IS}{\mathbb S}

\newcommand{\1}{1\hspace{-0.098cm}\mathrm{l}}

\newcommand{\id}{\mathrm{id}}

\renewcommand{\P}{{\mathbb P}}

\newcommand{\N}{{\mathbb N}}

\newcommand{\E}{{\mathbb E}}

\newcommand{\R}{{\mathbb R}}

\newcommand{\cN}{{\mathcal N}}

\newcommand{\Hess}{\text{Hess}\,}
\DeclareMathOperator*{\argmin}{arg\,min}

\newcommand{\vertiii}[1]{{\left\vert\kern-0.25ex\left\vert\kern-0.25ex\left\vert #1 
		\right\vert\kern-0.25ex\right\vert\kern-0.25ex\right\vert}}

\newcommand{\retr}{\operatorname{retr}}
\newcommand{\Par}{\operatorname{Par}}
\newcommand{\Proj}{P}
\newcommand{\proj}{\operatorname{proj}}

\theoremstyle{plain}
\newtheorem{theorem}{Theorem}[section]
\newtheorem{proposition}[theorem]{Proposition}
\newtheorem{lemma}[theorem]{Lemma}

\newtheorem{definition}[theorem]{Definition}

\theoremstyle{definition}
\newtheorem{remark}[theorem]{Remark}
\newtheorem{example}[theorem]{Example}

\begin{document}
		
	\title[Stochastic modified flows for RSGD]{Stochastic modified flows for Riemannian Stochastic Gradient Descent}
	
	\author[]%[Gess]
	{Benjamin Gess}
	\address{Benjamin Gess\\
		Institute of Mathematics\\
		TU Berlin\\
		10623 Berlin\\
		Germany\\
		and
		Max Planck Institute for Mathematics in the Sciences\\
		Inselstrasse 22\\
		04103 Leipzig\\ Germany}
	\email{benjamin.gess@gmail.com}
	
	\author[]%[Kassing]
	{Sebastian Kassing}
	\address{Sebastian Kassing\\
		Faculty of Mathematics\\
		University of Bielefeld\\
		Universit\"atsstraße 25\\
		33615 Bielefeld\\
		Germany}
	\email{skassing@math.uni-bielefeld.de}
	
		\author[]%[Rana]
	{Nimit Rana}
	\address{Nimit Rana\\
		Department of Mathematics\\
		University of York\\
		Heslington, York YO10 5DD\\
		UK}
	\email{nimit.rana@york.ac.uk}
	
\keywords{Riemannian stochastic gradient descent; diffusion approximation; supervised learning; weak error; Riemannian gradient flow}
\subjclass[2020]{Primary 62L20; Secondary 58J65, 60J20, 65K05}

% 62L20 Stochastic approximation
% 60 Markov processes
% 60J20 Applications of Markov proccesses
% 65C05 Monte Carlo methods
% 60J60 Diffusion processes

% 65K05 Numerical mathematical programming methods
%60H10 SDEs

% 58J65 Diffusion processes and stochastic analysis of manifolds

\begin{abstract} We give quantitative estimates for the rate of convergence of Riemannian stochastic gradient descent (RSGD) to Riemannian gradient flow and to a diffusion process, the so-called Riemannian stochastic modified flow (RSMF). Using tools from stochastic differential geometry we show that, in the small learning rate regime, RSGD can be approximated by the solution to the RSMF driven by an infinite-dimensional Wiener process. The RSMF accounts for the random fluctuations of RSGD and, thereby, increases the order of approximation compared to the deterministic Riemannian gradient flow. The RSGD is build using the concept of a retraction map, that is, a cost efficient approximation of the exponential map, and we prove quantitative bounds for the weak error of the diffusion approximation under assumptions on the retraction map, the geometry of the manifold, and the random estimators of the gradient.
\end{abstract}

\maketitle

\section{Introduction}
Many optimization problems of the form
\begin{align} \label{eq:intro}
	\text{Find } x^* \in \argmin_{y \in M} f(y)
\end{align}
are posed on non-linear sets $M$, e.g. $M$ being a subset of a Euclidean space due to non-linear constraints.  
For example, principal component analysis (PCA) leads to optimization problems on the Grassmann manifold or the Stiefel manifold, see Section~\ref{sec:PCA}. In supervised learning problems with positive homogeneous activation function, such as ReLU, one can restrict to networks with normalized weights, leading to optimization on the sphere~\cite{salimans2016weight, hinton2012improving, dereich2022minimal}, see Section~\ref{sec:weightnorm}. Other examples include learning hierarchical representations~\cite{chamberlain2017neural, nickel2017poincare, wilson2018gradient}, e.g. in language models, where the optimization is often performed on hyperbolic space, see Section~\ref{sec:hyperbolic}, and optimization over a parametrized family of probability measures, e.g.~in the training of generative adversarial networks \cite{goodfellow2014generative, salimans2018improving, shen2020sinkhorn}, see Section~\ref{sec:varinf}.

In all of these examples, the search space $M$ forms a Riemannian manifold. Therefore, popular optimization schemes for numerically solving (\ref{eq:intro}), such as stochastic gradient descent (SGD), have been transferred to the Riemannian setting~\cite{bonnabel2013, zhang2016first, zhang2016riemannian}.

The analysis of the dynamics of SGD, including its algorithmic, implicit bias, and the empirically observed good generalization properties of artificial neural networks trained by SGD, is highly involved. Since there exists a large toolbox from optimal control theory and stochastic analysis for the investigation of continuous time processes that is difficult to apply in discrete time, one is led to the derivation of continuous time limits of SGD. In this work, we present continuum limits for Riemannian SGD on manifolds in the small learning rate regime and provide quantitative estimates on the rate of convergence.

In~\cite{shah2021stochastic, karimi2022riemannian} the ODE method for Riemannian SGD is introduced proving that, as the learning rate tends to zero, the dynamics of SGD can be approximated by the solution to the Riemannian gradient flow ODE.
As the first main result, we give the first quantified bounds on the weak error of this approximation, see Theorem~\ref{thm:intro1}.

Note that the deterministic Riemannian gradient flow describes the typical behavior of SGD in the small learning rate regime without taking into account the randomness of the gradient estimators. 
For the second main result, we introduce a class of stochastic differential equations (SDEs) on $M$, which we call Riemannian stochastic modified flow (RSMF), that capture both the mean behavior of the dynamical system as well as its random fluctuations. This carefully chosen limiting object is proved to capture the dynamics of SGD more precisely, giving a higher order approximation in the weak sense compared to the limiting ODE. In the Euclidean setting similar SDEs have been introduced in~\cite{li2017stochastic, LiTaiE2019} with extensions to the overparametrized, infinite-particle limit given in~\cite{gess2023stochastic}. 

Let us introduce the central objects of the present work. See Appendix~\ref{sec:geo} for an overview of the notation used in this work.
Let $M$ be a complete, connected $\mathcal C^\infty$-Riemannian manifold. With an eye on the applications detailed in Section~\ref{sec:exa}, note that we do not assume $M$ to be compact.
Let $(\Xi, \mathcal G, \vartheta)$ be a probability space such that $L^2((\Xi,\vartheta);\R)$ is separable and $\grad \tilde f: M\times \Xi \to TM$ be a function that satisfies for all $x \in M$
that $\E_{\vartheta}[\|\grad \tilde f(x,\xi) \|_x^2]<\infty$ and
\begin{align*}
	\E_\vartheta[\grad \tilde f (x,\xi)]=\grad f(x),
\end{align*}
where $\|\cdot \|_x$ denotes the norm on $T_x M$ given by the Riemannian metric. E.g., if $\vartheta$ is a rotation invariant probability measure on $\R^d$ with bounded second moment, where $d=\dim(M)$, the choice $\grad \tilde f(x,\xi)=\grad f(x)+\xi$ satisfies the conditions above (identifying $\R^d$ with $T_x M$). Formally, $\grad \tilde f(\cdot, \xi)$ is not required to be a gradient vector field for every $\xi \in \Xi$. We use this notation to highlight the fact that $\grad \tilde f$ is a random estimator of $\grad f$.

Let $(\Omega, \cF, \P)$ be a probability space, $(\xi_n)_{n \in \N}$ be an i.i.d. sequence of $\vartheta$-distributed random variables on $(\Omega, \cF, \P)$ and set $(\cF_n)_{n \in \N_0}= (\sigma (\xi_1, \dots, \xi_n))_{n \in \N_0}$. For $\eta>0$ and $x \in M$ we consider an $(\cF_n)_{n\in \N_0}$-adapted, $M$-valued, process $(Z_n^\eta(x))_{n \in \N_0}$ satisfying $Z_0^\eta(x)=x$ and
\begin{align} \label{eq:SGD}
	Z_n^\eta(x) = \retr_{Z_{n-1}^\eta(x)} (-\eta \, \grad \tilde f(Z_{n-1}^\eta(x),\xi_n)), \quad \text{ for all } n \in \N,
\end{align}
where $\retr_{z}:T_z M \to M$ denotes the Riemannian exponential map $\exp_z: T_z M \to M$ at $z \in M$ or a computationally efficient approximation of the exponential map, see Definition~\ref{def:retraction}. We call $(Z_n^\eta(x))_{n \in \N_0}$ the \emph{Riemannian SGD scheme} or \emph{RSGD scheme} with \emph{learning rate} $\eta$, started at $x$. Note that, since $M$ is (geodesically) complete, $\exp_z$ is defined on the whole tangent space $T_z M$. Thus, the expression on the right-hand side of (\ref{eq:SGD}) is well-defined.

For small learning rates $\eta>0$ the dynamics of $(Z_n^\eta(x))_{n \in \N_0}$ can be compared to continuous time processes using the numerical time-scale $t_n=n \eta$. Since, in expectation, SGD performs an Euler-step for the gradient flow ODE
\begin{align} \label{eq:ODE}
	\dot z_t=-\grad f(z_t),
\end{align}
it seems natural to compare the dynamics of RSGD in the small learning rate regime with those of a solution $(z_t(x))_{t \ge 0}$ of (\ref{eq:ODE}) with initial condition $z_0(x)=x$. We quantify the quality of the approximation in the following theorem.

\begin{theorem} [See Theorem~\ref{thm:main1}] \label{thm:intro1}
	Assume that $M$ has bounded curvature, $\retr:TM \to M$ is an appropriate approximation of $\exp$ and $\grad \tilde f: M \times \Xi \to TM$ is sufficiently regular .
	Then, for all $T>0$ and sufficiently regular test functions $g:M \to \R$ there exists a constant $C>0$ such that for all $\eta >0$
	\begin{align*}
		\sup_{x \in M} \max_{n=0, \dots, \lfloor T/\eta \rfloor } |\E[g(Z_n^\eta(x))]-g(z_{n\eta}(x))|\le C \eta.
	\end{align*}
\end{theorem}
In light of Theorem~\ref{thm:intro1}, we say that the solution to the gradient flow is a weak order 1 approximation of RSGD. 

In order to get a continuous-in-time approximation of SGD that is of order $\cO(\eta^2)$ one has to introduce a diffusion term that takes into account the random fluctuations of RSGD. We introduce a class of diffusion processes on a different probability space and compare their marginal distributions to those of RSGD.
Let $(\tilde \Omega, (\tilde \cF_t)_{t \ge 0}, \tilde \cF, \tilde \P)$ be a complete, filtered probability space with right-continuous filtration $(\tilde \cF_t)_{t \ge 0}$.
We consider the solution to the SDE taking values in $M$
\begin{align} \label{eq:SDEintro} 
	dX_t^\eta(x)= B^\eta(X_t^\eta(x)) \, dt +  G^\eta(X_t^\eta(x),\cdot) \circ dW_t
\end{align}
with initial condition $X_0^\eta(x)=x$,
where $(W_t)_{t \ge 0}$ denotes a cylindrical Wiener process on $L^2((\Xi,\vartheta);\R)$ defined on $(\tilde \Omega, (\tilde \cF_t)_{t \ge 0}, \tilde \cF, \tilde \P)$, see Section~\ref{sec:SDE}. As coefficients we choose
\begin{align} \label{eq:Gintro}
	G(x, \xi):= G^\eta(x, \xi)  :=\sqrt{\eta}(\grad f(x)-\grad \tilde f(x,\xi))
\end{align}
and
\begin{align} \label{eq:Bintro} 
	B(x):= B^\eta(x) :=-\grad f(x)-\frac 12 \eta \Bigl(  \nabla_{\grad f(x)}(\grad f)+\int_\Xi \nabla_{\bar G(x,\xi)}\bar G(\cdot, \xi)\, \vartheta (d\xi)\Bigr),
\end{align}
where $\nabla$ denotes the Riemannian connection on $M$ and $\bar G= \frac {1}{\sqrt \eta} G$. We call the solution to (\ref{eq:SDEintro}) \emph{Riemannian stochastic modified flow} (RSMF). 

Additionally to the diffusion term (\ref{eq:Gintro}), we have to introduce two correction terms in (\ref{eq:Bintro}) in order to get a weak order 2 approximation of RSGD: a bias correction term that compensates the second order term in the Euler discretization of the gradient flow and a term that accounts for using the Stratonovich formulation in~(\ref{eq:SDEintro}), which is the more natural choice for defining SDEs on manifolds.

Let us state an informal version of the second main result of this article.

\begin{theorem}[see Theorem~\ref{theo:main2}] \label{thm:intro2}
	Assume that $M$ and $\grad \tilde f: M \times \Xi \to TM$ are sufficiently regular and $\retr:TM \to M$ is an appropriate approximation of $\exp$.
	Then, for all $T>0$ and sufficiently regular test functions $g: M \to \R$ there exists a constant $C>0$ such that for all $\eta >0$
	\begin{align*}
		\sup_{x \in M} \max_{n=0, \dots, \lfloor T/\eta \rfloor } |\E[g(Z_n^\eta(x))]-\tilde \E[g(X_{n\eta}^\eta(x))]| \le C \eta^2,
	\end{align*}
	where $(X_{t}^\eta(x))_{t \ge 0}$ denotes the unique solution to \eqref{eq:SDEintro}.
\end{theorem}
If $M$ is compact then the assumptions on the geometry of $M$ are satisfied. Moreover, if $M$ is the Euclidean space one can choose $\retr_x(v)=x+v$ for every $x \in M$ and $v \in T_x M \simeq \R^d$ and recover the results in~\cite[Theorem~9]{LiTaiE2019} and~\cite[Corollary~14]{gess2023stochastic} for the Stratonovich formulation of the Euclidean stochastic modified flow. However, Theorem~\ref{thm:intro2} includes unbounded manifolds that satisfy uniform boundedness conditions on the geometry, see Definition~\ref{def:BG}. We verify the assumptions on the geometry of the manifold and define appropriate approximations of the exponential map for principle component analysis, see Section~\ref{sec:PCA}, and for weight normalization in artificial neural networks, see Section~\ref{sec:weightnorm}.

The proofs of Theorem~\ref{thm:intro1} and Theorem~\ref{thm:intro2} proceed by a precise analysis of the (Markov) semigroups corresponding to the ODE (\ref{eq:ODE}) and the SDE (\ref{eq:SDEintro}) and their flow maps $(x,t) \mapsto z_t(x)$ and $(x,t) \mapsto X_t^\eta(x)$. We show regularity results for the flow maps and give quantitative bounds on their derivatives with respect to the initial condition. These bounds depend on the regularity of the random vector field $\grad \tilde f$ as well as on the curvature of $M$. 
Compared to the Euclidean case, the non-explosion of the solution to (\ref{eq:SDEintro}) does not follow from the Lipschitz-continuity of the coefficients alone. Therefore, on non-compact manifolds one has to be especially careful to ensure that the solution to (\ref{eq:SDEintro}) does not explode in finite time and that there exists a global flow map, see e.g.~\cite{elworthy1982stochastic, li1994properties, li1994strong}.

Note that our proofs do not use the fact that $\grad f$ is a gradient vector field and $\grad f$ can be replaced by any sufficiently regular vector field. In that sense, Theorem~\ref{thm:intro1} and Theorem~\ref{thm:intro2} naturally extend to Riemannian stochastic approximation schemes for non-gradient vector fields, as well as accelerated optimization methods defined on the tangent bundle $TM$ (see also Theorem~14 and~16 in~\cite{LiTaiE2019}).

The remainder of this article is organized as follows. In Section~\ref{sec:literature}, we give an overview of the existing literature on Riemannian stochastic gradient descent and continuous time approximations for SGD in the Euclidean and Riemannian setting. In Section~\ref{sec:prelim}, we define uniform retraction maps, see Definition~\ref{def:retraction}. In Section~\ref{sec:order1}, we prove the first main result, Theorem~\ref{thm:intro1}, by analyzing the flow of a vector field as well as the dynamics of RSGD. In Section~\ref{sec:SDE}, we introduce SDEs on manifolds driven by a cylindrical Wiener process. We give an existence and uniqueness result for locally Lipschitz continuous coefficients. Moreover, under additional regularity assumptions on the manifold we prove the strong completeness of the SDE (\ref{eq:SDEintro}), i.e. the existence of a global flow map, and give bounds for the derivatives of the corresponding Markov semigroup. In Section~\ref{sec:order2}, we prove the second main result, Theorem~\ref{thm:intro2}. Finally, in Section~\ref{sec:exa} we comment on the geometry in principle component analysis, weight normalization, hyperbolic space and statistical manifolds and give appropriate retraction maps in the respective optimization tasks. See Appendix~\ref{sec:geo} for the notation used throughout the article and Appendix~\ref{sec:higher} for an introduction into derivatives of higher order for scalar-valued functions and vector fields.

We refer the reader to \cite{do1992riemannian} for a more detailed introduction into the general theory of Riemannian manifolds, to \cite{udriste1994convex, absil2009optimization, boumal2023introduction} for an introduction into optimization on Riemannian manifolds and to~\cite{elworthy1982stochastic,hsu2002stochastic,ikeda1989stochastic} for an introduction into SDEs on manifolds.

\subsection{Overview of the literature} \label{sec:literature}
Diffusion approximations of Euclidean SGD in the small learning rates regime have been introduced by Li, Tai and E in~\cite{li2017stochastic} and~\cite{LiTaiE2019}. 
Following these original papers several results were derived for SDE approximations of SGD, e.g. generator based proofs~\cite{feng2018semigroups, hu2019}, approximations for SGD without reshuffling~\cite{ankirchner2022towards} and uniform-in-time estimates for strongly convex objective functions~\cite{feng2020uniform, lei2022}. Feng et al. presented a diffusion approximation for SGD performed on the sphere~\cite{feng2018semigroups}. In~\cite{gess2023stochastic}, a Euclidean analog of the stochastic modified flow (\ref{eq:SDEintro}) has been proposed in order to approximate the multi-point motion of SGD. This work also presents an approximation for the dynamics of SGD in the small learning rate - infinite width scaling regime for overparametrized neural networks. 
A diffusion approximation result for SGD with time-dependent learning rate has been derived in~\cite{fontaine2021convergence}. In~\cite{schwarz2023efficient} the concept of second order retractions is used in order to approximate the Brownian motion on a Riemannian manifold in a cost-efficient way. This result generalizes the classical Donsker's theorem in the Riemannian setting, see~\cite{jorgensen1975central}. For a discussion on the validity of the diffusion approximation for finite (non-infinitesimal) learning rate see~\cite{li2021validity}.

The derivation of stochastic continuum limits of SGD has proven instrumental in the analysis of optimization dynamics in several regards. For example, one of the motivations for the diffusion approximation of SGD is to simplify the derivation of optimal hyperparameter schedules, e.g. for the learning rate \cite{li2017stochastic} or the batch-size \cite{zhao2022batch, perko2023unlocking}, using optimal control theory. Regarding the asymptotic behavior of SGD, diffusion approximations can be used for finding a Lyapunov function~\cite{gess2023convergence, bach2023systematic}, developing dynamical systems arguments~\cite{fehrman2020}, investigating the critical noise decay rate for the convergence property~\cite{dereich2022cooling} and analyzing the implicit bias of SGD~\cite{wojtowytschII,li2021happens}.

Regarding SGD on a Riemannian manifold, the convergence of the objective function $f: M \to \R$ and its corresponding Riemannian gradient under the classical Robbins-Monro conditions has been shown by Bonnabel~\cite{bonnabel2013} using either the exponential map or a retraction map. This analysis has been refined for Hadamard manifolds~\cite{sakai2023convergence} and Riemannian stochastic approximation schemes, where the practitioner is only able to simulate a biased estimator of the vector field~\cite{durmus2020convergence}. In~\cite{shah2021stochastic} and~\cite{karimi2022riemannian}, the ODE method for SGD with decreasing learning rates is transferred to the Riemannian setting. For SGD with constant learning rate $\eta>0$, \cite{durmus2021riemannian} considered the invariant measure of the Markov-chain $(Z_n^\eta(x))_{n \in \N_0}$ and its asymptotic behavior as $\eta \to 0$. Tripuraneni et al.~\cite{tripuraneni2018averaging} introduced a version of the Ruppert-Polyak averaging technique for the Riemannian setting to improve convergence rates for the approximation of an isolated stable minimum. \cite{criscitiello2019efficiently, sun2019escaping, hsieh2023riemannian} considered escaping saddle points for perturbed gradient descent methods. See also~\cite{zhang2016first, zhang2016riemannian} regarding stochastic optimization results for geodesically convex target functions.
%Li presented completeness~\cite{li1994properties} and strong completeness~\cite{li1994strong} results for SDEs defined on non-compact manifolds as well as bounds on the first derivative of the flow map with respect to the initial condition.

\section{Uniform Retractions} \label{sec:prelim}
In this section, we present the assumptions on the mapping $\retr:TM \to M$ used in the definition of the RSGD scheme, see (\ref{eq:SGD}).
%For an introduction into the central objects of Riemannian optimization and the notation used throughout this article see Appendix~\ref{sec:geo}.
In many applications, the exponential map is difficult to compute and it is more cost efficient to work with an approximation, a so-called retraction map. We introduce the required assumptions on the retraction map in the following definition.

\begin{definition} \label{def:retraction}
	\begin{enumerate}
		\item [(i)]
		A $\cC^1$-map $\retr:TM \to M$ is called a \emph{retraction map} if for all $x \in M$ the restriction $\retr_x: T_x M \to M$ satisfies $\retr_x(0)=x$ and $D_0\retr_x:T_0(T_xM) \simeq T_x M \to T_x M$ is the identity map. A $\cC^2$-retraction map is called \emph{uniform first order retraction} if there exists a constant $C\ge 0$ such that for all $x \in M$, $t \ge 0$ and $v \in T_x M$ with $\|v\|=1$ one has
		\begin{align*}
			\|\dot \gamma_t\| \vee \Bigl\|\frac{\nabla}{dt} \dot \gamma_t\Bigr\| \le C,
		\end{align*}
		where $(\gamma_s)_{s \ge 0} = (\retr_x(sv))_{s \ge 0}$.
		\item[(ii)]
		A $\cC^2$-retraction map is called \emph{second order retraction} if for all $x \in M$ and $v \in T_x M$ one has
		\begin{align*}
			\frac{\nabla}{dt}\einschraenkung_{t=0} \, \dot \gamma_t=0,
		\end{align*}
		where $(\gamma_s)_{s \ge 0} = (\retr_x(sv))_{s \ge 0}$.
		Moreover, a $\cC^3$-second order retraction is called \emph{uniform second order retraction}  if there exists a constant $C\ge 0$ such that for all $x \in M$, $t \ge 0$ and $v \in T_x M$ with $\|v\|=1$ one has
		\begin{align*}
			\|\dot \gamma_t\| \vee \Bigl\|\frac{\nabla}{dt} \dot \gamma_t\Bigr\| \vee \Bigl\|\frac{\nabla^2}{dt^2} \dot \gamma_t\Bigr\| \le C.
		\end{align*}
		%		where $\frac{\nabla^2}{dt^2} \dot \gamma_t = \frac{\nabla}{dt} \bigl( \frac{\nabla }{dt} \dot \gamma_t \bigr)$.
	\end{enumerate}
\end{definition} 

\begin{example} \label{exa:secondretraction}
	\begin{enumerate}
		\item [(i)] The exponential map is a uniform second order retraction. In fact, for all $x \in M$, $t \ge 0$ and $v \in T_xM$ one has $\|\dot \gamma_t\|= \|v\|$ and $\frac{\nabla}{dt}\dot \gamma_t = 0$, where $(\gamma_t)_{t \ge 0}=(\exp_x(tv))_{t \ge 0}$.
		\item [(ii)] The stereographic projection is a uniform second order retraction for the unit sphere $\mathbb S^2 \subset \R^3$. For $p=(0,0,-1)$ we define
		$\retr_p: T_p M \simeq \R^2 \to \mathbb S^2$ via
		\begin{align*}
			\retr_p(x,y) = \Bigl( \frac{x}{1+\frac 14 x^2+\frac 14 y^2},\frac{y}{1+\frac 14 x^2+\frac 14 y^2}, \frac{-1+\frac 14 x^2+\frac 14 y^2}{1+\frac 14 x^2+\frac 14 y^2} \Bigr).
		\end{align*}
		Then, $\retr_p(0,0)=p$ and
		\begin{align*}
			D_0\retr_p = \begin{pmatrix}
				1 & 0 & 0 \\
				0 & 1 & 0
			\end{pmatrix}.
		\end{align*}
		To show that $\retr$ is a uniform second order retraction, by symmetry, it is sufficient to consider $\frac{\nabla}{dt}\gamma_t$, where $\gamma_t=\retr_p(tv)$ and $v=(1,0)$. We get
		\[
		\frac{d^2}{dt^2} \gamma_t = \Bigl( \frac{8t(t^2-12)}{(t^2+4)^3},0, \frac{64-48t^2}{(t^2+4)^3} \Bigr),
		\]
		so that $\frac{\nabla}{dt}\einschraenkung_{t=0} \dot \gamma= \Proj_p (\frac{d^2}{dt^2}\einschraenkung_{t=0} \gamma) = 0$. Here, $\Proj_p$ denotes the orthogonal projection onto $T_p M \simeq \R^2\times\{0\}$.  Moreover, $\|\dot \gamma_t\|$, $\|\frac{\nabla}{dt}\dot \gamma_t\|$ and $\|\frac{\nabla^2}{dt^2}\dot \gamma_t\|$ are uniformly bounded.
	\end{enumerate}
\end{example}

In general, let $M \subset \R^N$ be a smooth submanifold and set
\begin{align} \label{eq:proretraction}
	\retr_x(v) = \proj(x+v) , \quad x \in M, v \in T_x M,
\end{align}
where $\proj: \R^N \to M$ denotes a metric projection, i.e.
\begin{align} \label{eq:metricpro2}
	\proj(z) \in \argmin_{y \in M} d(y,z), \quad z \in \R^N.
\end{align}
By \cite[Theorem~1]{leobacher2021existence}, there exists an open set $U \supset \R^N$ containing $M$ such that for every $z \in U$ the minimizer in \eqref{eq:metricpro2} is unique and  $\proj\einschraenkung_U: U \to \R^N$ is $C^\infty$. Moreover, Theorem~4.9 in~\cite{absil2012projection} shows that for every $x \in M$ and $v \in T_x M$ it holds that $D_0 \retr_x v=v$ and $\frac{\nabla}{dt}\einschraenkung_{t=0} \dot \gamma_t=0$, where $(\gamma_s)_{s \ge 0}=(\retr_x(sv))_{s \ge 0}$. 
We apply a cutoff function to the mapping $\retr$ defined in \eqref{eq:proretraction} in order to construct a uniform second order retraction for compact submanifolds.

\begin{lemma} \label{lem:secondorderretraction}
	Assume that $M$ is compact and let $K \subset \R^N$ be a compact set such that $M \subset K \subset U$, where $U$ is as above. Let $c: \R^N \times \R^N \to \R^N$ be a smooth, bounded function such that $c(x,0)=0$, $D_v c(x,0)=\operatorname{id}_{\R^N}$ and $D^2_v c(x,0)=0$ for all $x \in M$, $x+c(x,v) \in K$ for all $x \in M$ and $v \in T_x M \subset \R^N$, and $D_v^\alpha c$ is bounded on $M \times \R^N$ for all $\alpha=1, 2,3$. Then the function $\retr: TM \to M$ given by
	\begin{align*}
		\retr_x(v) = \proj (x+  c(x,v) ), \quad x \in M, v \in T_x M,
	\end{align*}
	is a uniform second order retraction.
\end{lemma}

\begin{proof}
	For all $x \in M$ we denote by $\Proj_x \in \R^{N \times N}$ the matrix that corresponds to the orthogonal projection from $\R^N$ onto the tangent space $T_{\proj(x)}M$ of $M$ at $\proj(x)$ which is a $C^\infty$-smooth function on $U$.
	
	For $x \in M$ and $v \in T_x M$ with $\|v\|=1$ let $(\gamma_t)_{t \ge 0}=(\proj (x+  c(x,tv ))_{t \ge 0}$. Then, for all $t \ge 0$
	\begin{align*}
		\dot \gamma_t = D\proj (x+  c(x,tv ))  D_vc(x,tv)v,
	\end{align*}
	where $D_v$ denotes the Jacobi matrix w.r.t. the second argument in $c$. Hence, $\dot \gamma_0 = D\proj(x) v$. Moreover, 
	\begin{align*}
		\frac{\nabla}{dt} \dot \gamma_t 
		=&  \Proj_{\gamma_t}\Bigl( \bigl(D^2\proj (x+  c(x,tv ))\bigr)  (D_vc(x,tv)v,D_vc(x,tv)v) \Bigr)\\
		&+ \Proj_{\gamma_t}\Bigl( D\proj (x+  c(x,tv )) D_v^2c(x,tv)(v,v)  \Bigr),
	\end{align*}
	so that $\frac{\nabla}{dt} \dot \gamma_0 =  \Proj_{x}((D^2\proj (x))  (v,v))$. Here we used the fact that for a vector field $V$ along $\gamma$ one has $\frac{\nabla}{dt} V(t) = P_{\gamma_t} (\frac{d}{dt} V(t))$, where $P_{\gamma_t}$ denotes the orthogonal projection onto $T_{\gamma_t} M \subset \R^N$ and, for all $t$, $V(t)$ is identified as an element of $\R^N$ (see e.g. Proposition~5.3.2 in~\cite{absil2009optimization}). Thus, using~\cite[Theorem~4.9]{absil2012projection}, $\retr$ is a second order retraction.
	Lastly, using the boundedness of $D\Proj$ on $K$ as well as, for $\alpha=1, 2,3$, the boundedness of $D^\alpha \proj$ on $K$ and $D_v^\alpha c$ on $M \times \R^N$ we have that $\|\dot \gamma_t\|$, $\|\frac{\nabla}{dt} \dot \gamma_t\|$ and $\|\frac{\nabla^2}{dt^2} \dot \gamma_t\|$ are uniformly bounded in $x \in M$, $t \ge 0$ and $v \in T_x M$ with $\|v\|=1$.
\end{proof}

\begin{remark}
	One can choose $c(x,v)=v$ for all $x\in M$ and $v \in T_x M$ that satisfy $x+v \in K'$, where $K' \subset \R^N$ is a compact set such that $M \subset K' \subset U' \subset K$ for an open set $U'$. Therefore, if for $\vartheta$-almost all $\xi \in \Xi$ we have $x+\grad \tilde f(x,\xi) \in K'$ for all $x \in M$, we can without loss of generality assume that the second order retraction (\ref{eq:proretraction}) is a uniform second order retraction.
\end{remark}

\section{Order 1 Approximation} \label{sec:order1}
In this section, we quantify the weak approximation error for comparing the dynamics of Riemannian SGD with the solution to the gradient flow ODE
\begin{align} \label{eq:8236615} 
	\dot z_t(x) = - \grad f(z_t(x))
\end{align}
with initial condition $z_0(x)=x$. Recall that, for $x \in M$, $R(x): (T_x M)^3 \to T_x M$ denotes the curvature of $M$ at $x$ given by
\begin{align*}
	R(x)(u,v)w = (R(U,V)W)(x),
\end{align*}
where $U,V,W \in \mathfrak X^\infty(M)$ with $U(x)=u$, $V(x)=v$ and $W(x)=w$, see Appendix~\ref{sec:geo}. Moreover,
\begin{align*}
	\|R(x)\| = \sup_{v,w,u \in T_x M} \frac{\|R(x)(u,v)w\|}{\|u\| \, \|v\|\, \|w\|}.
\end{align*}
We next state the main result of this section, proving that the ODE (\ref{eq:8236615}) is an order $1$ approximation of RSGD if $M$ has bounded curvature and $f: M \to \R$ is sufficiently regular.

\begin{theorem} \label{thm:main1}
	Let $\retr:TM \to M$ be a uniform first order retraction and $\grad \tilde f: M \times \Xi \to TM$ be a function that satisfies for all $x \in M$
	\begin{align*}
		\E_\vartheta[\grad \tilde f(x,\xi)] = \grad f(x) \quad \text{ and } \quad \sup_{y \in M} \E_\vartheta[\|\grad \tilde f(y,\xi)\|^2]<\infty.
	\end{align*}
	Assume that $\grad f \in \mathfrak X^2_b(M)$ and $\sup_{x \in M}\|R(x)\|<\infty$. Then 
	for all $T \ge 0$ and $g \in \cC^2_b(M)$ there exists a constant $C\ge 0$ such that for all $\eta >0$
	\begin{align*}
		\sup_{x \in M} \sup_{n= 0, \dots, \lfloor T / \eta \rfloor} |\E[g(Z_n^\eta(x))]- g(z_{n \eta}(x))| \le C \eta.
	\end{align*}
\end{theorem}
For the proof of Theorem~\ref{thm:main1}, we first show regularity results for the flow of a vector field, in Section~\ref{sec:gradientflow}, and, afterwards, compare a single iteration step of RSGD with running the ODE for time $\eta$, in Section~\ref{sec:proof1}.

\subsection{Flow of a vector field} \label{sec:gradientflow}
Let $V \in \mathfrak X^1_b(M)$ and consider the ODE
\begin{align} \label{eq:47382}
	\dot z_t(x)=V(z_t(x))
\end{align}
with initial condition $z_0(x)=x$. Then, for every $x \in M$ there exists a unique solution $(z_t(x))_{t \in \R}$ of (\ref{eq:47382}) that does not explode in finite time.
%, since $d(x_t(x),x_0(x))\le |t| \|V(x)\|_{\mathfrak X^0_b(M)}$.
Moreover the mapping $M \times \R \ni (x,t) \mapsto z_t(x) \in M$ is $C^1$, see e.g. \cite[Theorem~B.3]{duistermaat2000lie}. We give a quantitative bound on the first derivative with respect to the initial condition. This follows immediately from \cite[Lemma~3.4]{do1992riemannian} and Gronwall's inequality.

\begin{lemma} \label{lem:ODE1}
	For $x \in M$ and $v \in T_x M$, $(D_x z_t(v))_{t \ge 0}$ satisfies the differential equation
	\begin{align*}
		\frac{\nabla}{dt} (D_x z_t(v)) = \nabla_{D_x z_t(v)} V.
	\end{align*}
	Moreover, for all $t\ge 0$ we have
	$
	\|D_x z_t(v)\| \le \|v\| \exp({\|V(x)\|_{\mathfrak X^1_b(M)} t}).
	$
\end{lemma}

%\begin{proof}
%	First, note that $(D_x x_t(v))_{t \ge 0}$ is a vector field along $(x_t(x))_{t \ge 0}$, since for all $t \ge 0$ we have $D_x x_t(v) \in T_{x_t(x)}M$. Without loss of generality we can assume that $M$ is a submanifold of $\R^N$ for some $N \in \N$. Then, 
%	\begin{align*}
	%		\frac{\nabla}{dt} (D_x x_t(v)) & = \Proj_{x_t(x)} \Bigl(\frac{d}{dt} D_x x_t(v)\Bigr)
	%		= \Proj_{x_t(x)} \Bigl(\frac{d}{dt} \frac{d}{dv} x_t(x)\Bigr) \\
	%		&= \Proj_{x_t(x)} \Bigl(\frac{d}{dv} \frac{d}{dt}  x_t(x)\Bigr)
	%		= \Proj_{x_t(x)} \Bigl(\frac{d}{dv} V(x_t(x))\Bigr) \\	
	%		&=  \nabla_{D_x x_t(v)} V,
	%	\end{align*}
%	where $\Proj_{x_t(x)}$ denotes the projection onto $T_{x_t(x)}M$.
%	To prove the second statement, one has $D_x x_0(v)=v$,
%	$$
%		\frac{d}{dt}\|D_x x_t(v)\|^2 = 2 \langle D_x x_t(v), \nabla_{D_x x_t(v)} V \rangle 
%	$$
%	and $\|\nabla_{D_x x_t(v)} V\| \le  \|D_x x_t(v)\| \, \| V\|_{\mathfrak X^1_b(M)}$ so that (\ref{eq:22364262}) follows from Gronwall's inequality.
%\end{proof}

Next, we consider quantitative bounds on the second derivative of the flow. For this estimate we need a bound on the Riemannian curvature.

\begin{lemma} \label{lem:ODE2}
	Assume that $V \in \mathfrak X_b^2(M)$ and $\sup_{x \in M}\|R(x)\|<\infty$. Then the mapping $(x,t) \mapsto z_t(x)$ is $\cC^2$ and for all $T \ge 0$ there exists a constant $C\ge 0$ that only depends on $T$, $\|V\|_{\mathfrak X^2_b(M)}$ and $\sup_{x \in M}\|R(x)\|$ such that for all $0 \le t \le T$, $s \ge 0$, $x \in M$ and $v,w \in T_x M$ with $\|v\|=\|w\|=1$ we have 
	\begin{align*}
		\Bigl\| \frac{\nabla}{ds} D_{\gamma_s} z_t(v_s) \Bigr\| \le C,
	\end{align*}
	where
	$(\gamma_s)_{s \in \R}=(\exp_x(sw))_{s \in \R}$ and $(v_s)_{s \in \R}$ is given by $v_s = \Par_{\gamma\einschraenkung_{[0,s]}}v$ for $s\ge 0$ and $v_s = (\Par_{\gamma\einschraenkung_{[s,0]}})^{-1}v$ for $s < 0$.
\end{lemma}

\begin{proof}
	For the statement that $(x,t) \mapsto z_t(x)$ is $\cC^2$ see \cite[Theorem~B.3]{duistermaat2000lie}.
	Let $x \in M$, $v,w \in T_x M$ with $\|v\|=\|w\|=1$, $(\gamma_s)_{s \in \R}=(\exp_x(sw))_{s \in \R}$ and $(v_s)_{s \in \R}$ be given by $v_s = \Par_{\gamma\einschraenkung_{[0,s]}}v$ for $s\ge 0$ and $v_s = (\Par_{\gamma\einschraenkung_{[s,0]}})^{-1}v$ for $s < 0$. Then $\R^2 \ni (s,t) \mapsto z_t(\gamma_s) \in M$ is a parametrized surface and $(s,t) \mapsto D_{\gamma_s} z_t(v_s)$ is a vector field along this surface. Using \cite[Lemma~4.1]{do1992riemannian} and Lemma~\ref{lem:ODE1}, we get
	\begin{align*}
		\frac{\nabla}{dt}\frac{\nabla}{ds} D_{\gamma_s} z_t(v_s) %= &\frac{\nabla}{ds} \frac{\nabla}{dt}  D_{\gamma_s} z_t(v_s) + R(z_t(\gamma_s))( D_{\gamma_s} z_t(\dot \gamma_s) , V(z_t(\gamma_s))) D_{\gamma_s} z_t (v_s) \\
		= &\frac{\nabla}{ds} \bigl( \nabla_{D_{\gamma_s} z_t(v_s)} V \bigr) + R(z_t(\gamma_s))( D_{\gamma_s} z_t(\dot \gamma_s) , V(z_t(\gamma_s))) D_{\gamma_s} z_t (v_s)\\
		= &\nabla_{\frac{\nabla}{ds} D_{\gamma_s}z_t(v_s)} V + (\nabla^2 V)(D_{\gamma_s} z_t(\dot \gamma_s), D_{\gamma_s} z_t(v_s)) \\
		&+ R(z_t(\gamma_s))( D_{\gamma_s} z_t(\dot \gamma_s) , V(z_t(\gamma_s))) D_{\gamma_s} z_t (v_s),
	\end{align*}
	see also Appendix~\ref{sec:higher}. Now, $\|\frac{\nabla}{ds} D_{\gamma_s} z_0(v_s)\|^2 = \|\frac{\nabla}{ds} v_s\|^2 = 0$ and 
	\begin{align*}
		\frac{d}{dt} \Bigl\|\frac{\nabla}{ds} D_{\gamma_s} z_t&(v_s)\Bigr\|^2 =  2 \langle \frac{\nabla}{ds} D_{\gamma_s} z_t(v_s), \frac{\nabla}{dt}\frac{\nabla}{ds} D_{\gamma_s} z_t(v_s) \rangle \\
		\le & 2 \Bigl\|\frac{\nabla}{ds} D_{\gamma_s} z_t(v_s)\Bigr\| \|V\|_{\mathfrak X^2_b(M)} \Bigl( \Bigl\|\frac{\nabla}{ds} D_{\gamma_s} z_t(v_s)\Bigr\| + (1+\|R\|_{\infty}) \|D_{\gamma_s} z_t\|^2 \Bigr)\\
		\le & 2 \Bigl( \Bigl\|\frac{\nabla}{ds} D_{\gamma_s} z_t(v_s)\Bigr\|^2 +1 \Bigr) \|V\|_{\mathfrak X^2_b(M)} \Bigl( 1 +(1+\|R\|_\infty ) \, \|D_{\gamma_s} z_t\|^2\Bigr),
	\end{align*}
	where $\|R\|_\infty := \sup_{x \in M} \|R(x)\|$.
	Using Lemma~\ref{lem:ODE1} and Gronwall's inequality we get
	\begin{align*}
		\Bigl\|\frac{\nabla}{ds} D_{\gamma_s} z_t(v_s)\Bigr\|^2 \le \Bigl( \int_0^t \alpha_s \, ds\Bigr) \,   e^{\int_0^t \alpha_s \, ds},
	\end{align*}
	where 
	$
	\alpha_s = 2 \|V\|_{\mathfrak X^2_b(M)} \bigl( 1 +(1+\|R\|_\infty ) \, \exp({2 \| V\|_{\mathfrak X^2_b(M)} t})\bigr).
	$
\end{proof}

With Lemma~\ref{lem:ODE1} and Lemma~\ref{lem:ODE2} at hand, we analyze the mapping $x\mapsto g(z_t(x))$ for a $g \in \cC_b^2$.

\begin{proposition} \label{prop:ODE}
	Assume that $\sup_{x \in M} \|R(x)\| < \infty$ and let $g \in \cC^2_b(M)$ and $V \in \mathfrak X^2_b(M)$. Then for all $T \ge 0$ there exists a constant $C\ge 0$ that only depends on $\|g\|_{\cC^2_b(M)}$, $\|V\|_{\mathfrak X^2_b(M)}$ and $\sup_{x \in M} \|R(x)\|$ such that for all $0 \le t \le T$ the function
	\begin{align*}
		\varphi_t:M \to \R \quad ; \quad x \mapsto g(z_t(x))
	\end{align*}
	satisfies $\|\varphi_t\|_{\cC^2_b(M)}\le C$.
\end{proposition}

\begin{proof}
	For all $t \in T$ we clearly have $\varphi_t \in \cC^2$ and $\|\varphi_t\|_{\cC^0_b} \le \|g\|_{\cC_b^0}$. 
	For $x \in M$ and $v \in T_x M$,
	\begin{align} \begin{split} \label{eq:23455242}
			\langle \grad \varphi_t(x), v \rangle  = D_x\varphi_t(v)  = (D_{z_t(x)}g) (D_x z_t(v)) = \langle \grad g(z_t(x)), D_x z_t(v) \rangle 
		\end{split}
	\end{align}
	so that, using Lemma~\ref{lem:ODE1}, we get 
	\begin{align*}
		\|\grad \varphi_t(x)\| \le \sup_{x \in M}\|\grad g(x)\| \, \sup_{x \in M}\|D_x z_t\| \le \|g\|_{\cC^1_b(M)}  \exp({\|V\|_{\mathfrak X^1_b(M)}t}).
	\end{align*}
	By polarization, it suffices to derive a bound for $\langle \Hess \varphi_t(x) v, v\rangle $ for all $x \in M$ and $v \in T_x M$ with $\|v\|=1$. Write $(\gamma_s)_{s \in \R} = (\exp_x(s v))_{s \in \R}$ and note that $\dot \gamma_s= \Par_{\gamma\einschraenkung_{[0,s]}} v$ for all $s\ge 0$.
	Using \cite[Proposition 5.5.4]{absil2009optimization} and~(\ref{eq:23455242}), we get
	\begin{align*}
		\langle \Hess \varphi_t(x) v,v \rangle &=  \frac{d^2}{ds^2}\einschraenkung_{s=0} \, \varphi_t (\exp_x(sv)) = \frac{d}{ds}\einschraenkung_{s=0} \, \langle \grad g(z_t(\gamma_s)), D_{\gamma_s} z_t(\dot \gamma_s) \rangle\\
		&=  \langle \nabla_{D_{x}z_t(v)} \grad g, D_x z_t(v) \rangle + \langle \grad g(z_t(x)), \frac{\nabla}{ds}\einschraenkung_{s=0} D_{\gamma_s} z_t(\dot \gamma_s) \rangle,
	\end{align*}
	so that the statement follows from applying Lemma~\ref{lem:ODE1} and Lemma~\ref{lem:ODE2}.
\end{proof}

\subsection{Proof of Theorem~\ref{thm:main1}} \label{sec:proof1}
Fix $T \ge 0$ and for $0 \le t \le T$ let $\varphi_t: M \to \R$ be given by $\varphi_t(x)=g(z_t(x))$. By Proposition~\ref{prop:ODE}, there exists a constant $C$ such that $\sup_{0 \le t \le T}\|\varphi_t\|_{\cC^2_b(M)}\le C$.

For $x \in M$ and $n \in \N$ we get, using the triangle inequality and the Markov property for $(Z_n^\eta(x))_{n \in \N_0}$,
\begin{align}
	\begin{split} \label{eq:8236523}
		|\E[g(Z_n^\eta(x))]- &g(z_{n \eta}(x))|
		= \Bigl| \sum_{i=1}^n \E\bigl[\varphi_{(n-i)\eta}(Z_i^\eta(x))- \varphi_{(n-i)\eta}(z_\eta(Z_{i-1}^\eta(x))) \bigr] \Bigr| \\
		&\le \sum_{i=1}^n \bigl|\E\bigl[ \E[\varphi_{(n-i)\eta}(Z_i^\eta(x)) | \mathcal F_{i-1}] -  \varphi_{(n-i)\eta}(z_\eta(Z_{i-1}^\eta(x))) \bigr]\bigr| \\
		&\le \sum_{i=1}^n \sup_{y \in M} \bigl| \E[\varphi_{(n-i)\eta}(Z_1^\eta(y))] -  \varphi_{(n-i)\eta}(z_\eta(y)) \bigr|.
	\end{split}
\end{align}

We next derive a uniform bound that holds for each individual summand in the latter sum.
For $t\ge 0$ and $\varphi \in \cC^2_b(M)$ we have
\begin{align*}
	&\varphi(z_t(x))= \varphi(x) -\int_0^t \langle  \grad \varphi(z_s(x)), \grad f(z_s(x)) \rangle \, ds \\
	&= \varphi(x) - t \langle \grad \varphi(x), \grad f(x) \rangle - \int_0^t \int_0^s \frac{d}{du} \langle  \grad \varphi(z_u(x)), \grad f(z_u(x)) \rangle \, du \,  ds.
\end{align*}
Using the fact that for all $u \ge 0$
\begin{align*}
	\frac{d}{du} \langle  \grad \varphi(z_u(x)), \grad f(z_u(x)) \rangle = &- \langle \Hess \varphi(z_u(x)) \,  \grad f(z_u(x)), \grad f(z_u(x)) \rangle \\
	&- \langle  \grad \varphi (z_u(x)), \nabla_{\grad f(z_u(x))} \grad f \rangle,
\end{align*}
as well as $\varphi \in \cC^2_b(M)$ and $\grad f \in \mathfrak X^2_b(M)$, there exists a constant $C\ge 0$ that only depends on $\|\varphi\|_{\cC^2_b(M)}$ and $\|\grad f\|_{\mathfrak X^1_b(M)}$ such that
\begin{align} \label{eq:225524141}
	\sup_{x \in M} |\varphi(z_t(x))-\varphi(x)+t\langle \grad \varphi(x),\grad f(x) \rangle | \le C t^2.
\end{align}

Next, we turn to the SGD scheme defined in (\ref{eq:SGD}).
For $x \in M$ and $\xi \in \Xi$ we let $(\gamma_t)_{t \in [0,1]}=(\retr_{x}(-\eta t \,  \grad \tilde f(x,\xi)))_{t \in [0,1]}$ and note that
\begin{align*}
	\varphi(\gamma_1)&= \varphi(x)+ \frac{d}{ds}\einschraenkung_{s=0} \varphi(\gamma_s) +\int_0^t \int_0^s  \frac{d^2}{du^2} \varphi(\gamma_u) \, \, du\, ds.
\end{align*}	
Since $\retr: TM \to M$ is a retraction, see Definition~\ref{def:retraction}, we get 
\[
\frac{d}{ds}\einschraenkung_{s=0} \varphi(\gamma_s)= \langle \grad \varphi(x), D_0\retr_x(-\eta \,  \grad \tilde f (x,\xi))\rangle = - \eta \langle \grad \varphi(x) , \grad \tilde f(x,\xi)\rangle.
\]
For the remainder, we note that for all $u \ge 0$
\begin{align*}
	\frac{d^2}{du^2} \varphi(\gamma_u) = \langle \Hess \varphi(\gamma_u)\dot \gamma_u, \dot \gamma_u \rangle + \langle \grad \varphi(\gamma_u), \frac{\nabla}{du} \dot \gamma_u \rangle.
\end{align*}
Using that $\retr$ is a uniform first order retraction, there exists a constant $C \ge 0$ such that $\|\dot \gamma_u\| \le C \eta \|\grad \tilde f(x,\xi)\|$ and $\|\frac{\nabla}{du} \dot \gamma_u\| \le C \eta^2 \|\grad \tilde f(x,\xi)\|^2$ for all $u \ge 0$. 
Thus, there exists a constant $C\ge 0$ that only depends on $\varphi$ via $\|\varphi\|_{\cC^2_b(M)}$ such that
% for all $u \ge 0$
%\begin{align*}
%\bigl| \frac{d^2}{du^2} \varphi(\gamma_u) \bigr| \le C \eta^2 \|\grad \tilde f(x,\xi)\|^2.
%\end{align*}
%Altogether, for a constant $C\ge 0$,
\begin{align} \label{eq:2222232}
	\sup_{x \in M} | \E[\varphi(Z_1^\eta(x))]-\varphi(x) + \eta \langle \grad \varphi(x), \grad f(x) \rangle| &\le C \eta^2   \Bigl( \sup_{x \in M} \E[\|\grad \tilde f(x,\xi_1)\|^2] \Bigr).
\end{align}
Combining (\ref{eq:225524141}) and (\ref{eq:2222232}), there exists a constant $C\ge0$ that only depends on $\|\varphi\|_{\cC^2_b(M)}$, $\|\grad f\|_{\mathfrak X_b^1(M)}$ and $\sup_{x \in M}\E_\vartheta[\|\grad \tilde f(x,\xi)\|^2]$ such that 
\begin{align} \label{eq:119283}
	\sup_{x \in M} | \E_\vartheta[\varphi(Z_1^\eta(x))] -\varphi(z_\eta(x)) | \le C \eta^2.
\end{align}
Using that $\sup_{0 \le t \le T}\|\varphi_t\|_{\cC^2_b(M)}<\infty$ and applying (\ref{eq:119283}) to each summand on the right-hand side of (\ref{eq:8236523}), there exists a constant $C\ge 0$ such that for all $n \in \{0, \dots, \lfloor T/\eta \rfloor\}$
\begin{align*}
	\sup_{x \in M} |\E[g(Z_n^\eta(x))]- g(z_{n \eta}(x))| \le n C \eta^2 \le C T \eta.
\end{align*}

\section{SDEs on Manifolds Driven by Cylindrical Wiener Processes} \label{sec:SDE}
In this section, we give well-posedness and regularity results for stochastic differential equations (SDEs) on a Riemannian manifold driven by a cylindrical Wiener process. For a more detailed introduction into stochastic analysis on Riemannian manifolds we refer the reader to~\cite{elworthy1982stochastic, ikeda1989stochastic,hsu2002stochastic}. Regarding SDEs on Euclidean space driven by a cylindrical Wiener process we refer the reader to~\cite{da2014stochastic, liu2015stochastic, riedle2011cylindrical}.

Let $(\Xi,\cG,\vartheta)$ be a measure space such that $\vartheta$ is a finite measure and the space $L^2((\Xi,\vartheta);\R)$ is separable. We denote by $\langle \cdot, \cdot \rangle_\vartheta$, resp. $\| \cdot\|_\vartheta$, the usual inner product, resp. its associated norm, on the space $L^2((\Xi,\vartheta);V)$, where $V$ denotes a finite-dimensional Hilbert space.

Let $(W_t)_{t \ge 0}$ be a cylindrical Wiener process on $L^2((\Xi,\vartheta);\R)$ defined on a filtered, complete probability space $(\tilde \Omega,\tilde \cF, (\tilde \cF_t)_{t\geq0}, \tilde \P)$ with right-continuous filtration $(\tilde \cF_{t})_{t\geq 0}$, that is,
\begin{enumerate}
	\item [(i)] for every $t\geq 0$, the map $W_t:L^2((\Xi,\vartheta);\R)\to L^2((\tilde \Omega,\tilde \P);\R)$ is linear;
	\item[(ii)] for every $h\in L^2((\Xi,\vartheta);\R)$, $(W_t(h))_{t \ge 0}$ is an $(\tilde \cF_t)_{t\geq 0}$-Brownian motion with $\Var (W_t(h))=\|h\|_{\vartheta}^2t$ for every $t\ge 0$.
\end{enumerate}
For an $(\tilde \cF_{t})_{t\geq 0}$-progressively measurable $L^2((\Xi,\vartheta);\R)$-valued process $(G(t,\cdot))_{t \ge 0}$ that almost surely satisfies $G \in L^2_{\operatorname{loc}}([0,\infty);L^2((\Xi, \vartheta);\R))$
we define
\[
\int_{ 0 }^{ t } \int_{ \Xi }   G(s,\xi) \,  W(d \xi,ds):=\int_{ 0 }^{ t } \Upsilon(s) \,  dW_s  ,
\]
where $\Upsilon(s)$ is given by $\Upsilon(s)h=\langle G(s,\cdot) , h \rangle_{\vartheta}$ for all $h \in L^2((\Xi,\vartheta);\R)$. For the definition of the integral with respect to a cylindrical Wiener process see, e.g.,~\cite[Section~2.2.4]{gawarecki2010stochastic}.
It is known that there exist an orthonormal basis of $L^2((\Xi,\vartheta);\R)$, $(e_i)_{i \in \N}$, and independent $\R$-valued $(\tilde \cF_t)_{t \ge 0}$-Brownian motions, $(W_t^{(1)})_{t \ge 0}, (W_t^{(2)})_{t \ge 0}, \dots$, such that
\begin{align}\label{eq:SDErow}
	\int_0^t \int_\Xi G(s, \xi) \,  W(d\xi,ds) = \sum_{i=1}^\infty \int_0^t \langle G(s,\cdot) , e_i\rangle_{\vartheta} \,  dW^{(i)}_s,
\end{align}
where the integrals on the right-hand side of (\ref{eq:SDErow}) are classical Itô-integrals with respect to $\R$-valued Brownian motions and the sum is almost surely finite, see e.g. \cite[Section~4.2.2]{da2014stochastic}. Moreover, 
\begin{align*}
	\|\Upsilon(s)\|_{\operatorname{HS}}^2 & = \sum_{i=1}^\infty |\Upsilon(s) e_i|^2 = \sum_{i=1}^\infty |\langle G(s,\cdot), e_i \rangle_\vartheta|^2 = \|G(s,\cdot)\|_\vartheta^2,
\end{align*}
where the left-hand side denotes the Hilbert-Schmidt norm of the operator $\Upsilon(s): L^2((\Xi,\vartheta);\R) \to \R$.

Next, we introduce the notion of a solution to an SDE on a manifold driven by a cylindrical Wiener process. This generalizes the approach in~\cite{hsu2002stochastic} for SDEs on manifolds driven by a finite-dimensional Brownian motion. We need some additional notation.

Let $\tilde L^2(M \times \Xi;TM)$ be the space of all  functions $G:M \times \Xi\to TM$ such that, for all $x \in M$ and $\xi\in \Xi$, $G(x, \xi) \in T_x M$ and 
\begin{align} \label{eq:333234}
	\|G(x,\cdot )\|_\vartheta^2 := \int_\Xi \|G(x, \xi')\|_x^2 \, \vartheta(d\xi')<\infty.
\end{align}
We denote by $\tilde {\mathfrak X}^0(M)$ the space of functions $G \in \tilde L^2(M \times \Xi; TM)$ such that for $\vartheta$-a.e. $\xi \in \Xi$ we have $G(\cdot,\xi) \in \mathfrak X^0(M)$ and for all $x \in M$ there exists a neighborhood $U \subset M$ such that $x \in U$ and 
\begin{align} \label{eq:909332}
	\int_\Xi \sup_{y \in U} \|G(y,\xi)\|_y^2 \, \vartheta(d\xi)<\infty.
\end{align}
Note that (\ref{eq:909332}) together with the dominated convergence theorem implies that for all $g \in \mathcal C^1(M)$ the mapping $x\mapsto Gg(x,\cdot) = \langle G(x,\cdot), \grad g\rangle \in L^2((\Xi,\vartheta);\R)$ is continuous w.r.t. the norm $\|\cdot \|_\vartheta$.
Moreover, we denote by $\tilde {\mathfrak X}^0_b(M)$ the space of functions $G \in \tilde{\mathfrak X}^0(M)$ with 
\begin{align*}
	\|G\|_{\tilde{\mathfrak X}^0_b(M)}:= \sup_{x \in M} \|G(x,\cdot)\|_\vartheta^2<\infty.
\end{align*}
Similarly, for $\alpha \in \N$ we denote by $\tilde {\mathfrak X}^\alpha(M)$ the space of functions $G \in \tilde{\mathfrak X}^0(M)$ such that for $\vartheta$-a.e. $\xi \in \Xi$ we have $G(\cdot,\xi)\in \mathfrak X^\alpha(M)$ and for all $1\le \beta\le \alpha$ and $x \in M$ there exists a neighborhood $U \subset M$ such that $x \in U$ and 
\begin{align*}
	\int_\Xi \sup_{y \in U} \|\nabla^\beta G(y,\xi)\|^2\, \vartheta(d \xi)<\infty,
\end{align*}
see Definition~\ref{def:bounded}.
Again, the dominated convergence theorem implies that for all $1\le \beta \le \alpha$, $V_1, \dots, V_\beta \in \mathfrak X^{\beta-1}(M)$ and $g \in \mathcal C^1(M)$ we have that
\[
x \mapsto (\nabla^\beta G(x,\cdot))(V_1, \dots, V_\beta)g \in L^2((\Xi,\vartheta);\R)
\]
is continuous w.r.t. the norm $\|\cdot\|_\vartheta$.
We analogously define the space $\tilde {\mathfrak X}^\alpha_b(M)$ and its respective norm.

Let $B\in \mathfrak X^1(M)$ and $G \in \tilde{\mathfrak X}^2(M)$ and consider the following formal SDE
\begin{align} \label{eq:SDEdefi}
	dX_t(x)= B(X_t(x)) \, dt + \int_\Xi G(X_t(x),\xi) \circ W(d\xi, ds)
\end{align}
with initial condition $X_0(x)=x \in M$. To make sense of this expression we consider the one-point compactification $\hat M=M \cup \{\partial_M\}$ of $M$ and apply test functions. We interpret the second summand in (\ref{eq:SDEdefi}) as the sum of an Itô-integral and a corresponding Stratonovich correction term.

\begin{definition}
	\begin{enumerate}
		\item[(i)] Let $(K_n)_{n \in \N}$ be an increasing sequence of compact subsets of $M$ satisfying $M = \bigcup_{n \in \N} K_n$. Such a sequence exists since differentiable manifolds are, by definition, second countable. For a continuous, $(\tilde {\mathcal F}_t)_{t \ge 0}$-adapted process $(X_t)_{t \ge 0}$ taking values in $\hat M$ we define 
		\begin{align*}
			e := \lim\limits_{n \to \infty}\tau_{n} \quad \text{ with } \quad \tau_n := \inf \{t \ge 0: X_t \notin K_n\},
		\end{align*}
		which is an $(\tilde {\mathcal F}_t)_{t \ge 0}$-stopping time and independent of the choice for $(K_n)_{n \in \N}$. We call $e$ the \emph{explosion time} of $(X_t)_{t \ge 0}$.
		\item[(ii)] 
		A continuous $\hat M$-valued, $(\tilde {\mathcal F}_t)_{t \ge 0}$-adapted process $(X_t(x))_{t \ge 0}$ with $X_0(x)=x$ and explosion time $e(x)$ is called \emph{solution to the SDE (\ref{eq:SDEdefi}) started in $x$} if for all $g \in \mathcal C^\infty (M)$ one has almost surely that for all $t \in [0,e(x))$
		\begin{align} \begin{split} \label{eq:SDEg}
				g(X_t(x)) =&g(x)+ \int_0^t Bg(X_s(x)) \, \rd s + \frac {1}{2} \int_0^t \int_\Xi  G Gg(X_s(x),\xi) \vartheta(d\xi) \, ds \\
				&+    \int_0^t\int_\Xi  Gg(X_s(x),\xi) \, W(ds,d\xi).
			\end{split}
		\end{align}
	\end{enumerate}
\end{definition}

We next present an existence and uniqueness statement for the SDE (\ref{eq:SDEdefi}). This result can be obtained as in the proof of Theorem~1.2.9 in \cite{hsu2002stochastic}, where SDEs on manifolds driven by a finite dimensional noise were considered. We take special attention to the required smoothness of the coefficients in (\ref{eq:SDEdefi}).

\begin{proposition} \label{prop:Kolmogorovcompact}
	Let $B\in \mathfrak X^1(M)$ and $G\in \tilde {\mathfrak X}^{2}(M)$. Then for every $x \in M$ there exists a solution $(X_t(x))_{t \ge 0}$ to the SDE \eqref{eq:SDEdefi} with initial condition $X_0(x)=x$ and explosion time $e(x)$ that is unique up to its explosion time. Moreover, $(X_t(x))_{t\ge 0}$ satisfies \eqref{eq:SDEg} for all $g \in \cC^2(M)$.
\end{proposition}

In the second main result, Theorem~\ref{theo:main2}, we will work with solutions to SDEs that do not explode in finite time and under the following assumption on the regularity of the corresponding Markov semigroup.

\begin{definition} \label{def:reg}
	Let $\alpha \in \{2,3, \dots\}$. We say that $M$ has regularity $\alpha$ if for all $B \in \mathfrak X_b^{\alpha}(M)$ and $G \in \tilde{\mathfrak X}^{\alpha+1}_b(M)$ we have that 
	\begin{enumerate}
		\item[(i)] the SDE \eqref{eq:SDEdefi} is complete, i.e. for all $x \in M$ we have $e(x)=\infty$ almost surely, where $e(x)$ is the explosion time for the SDE \eqref{eq:SDEdefi} started in $x$, and
		\item[(ii)] for all $g \in \mathcal C^\alpha_b(M)$ and $t \ge 0$ we have that
		$
		\Psi_t(x) := \tilde \E[g(X_t(x))] \in \mathcal C^\alpha_b(M).
		$
		Moreover, for all $T\ge 0$ there exists a constant $C\ge 0$ that only depends on $T$, $\|B\|_{\mathfrak X^\alpha_b(M)}$, $\|G\|_{\tilde{\mathfrak X}^{\alpha+1}_b(M)}$ and $\|g\|_{\mathcal C^\alpha_b(M)}$ such that 
		\[
		\sup_{t \in [0,T]}\|\Psi_t\|_{\mathcal C^\alpha_b(M)} \le C .
		\]
	\end{enumerate}
\end{definition}

Definition~\ref{def:reg} summarizes the necessary assumptions on the manifold in order to derive the order 2 approximation result in Theorem~\ref{theo:main2}.
Verifying the completeness for SDEs on Riemannian manifolds is more involved than simply using the Lipschitz-continuity of the coefficients as in the Euclidean case. It is known that the existence of a uniform cover \cite[Corollary~6.1]{elworthy1982stochastic} or a weak uniform cover \cite[Theorem~2.4]{li1994properties} implies completeness of the corresponding SDE. In the next section, we introduce a large class of manifolds, including all compact manifolds, that satisfy the assumptions in Definition~\ref{def:reg}.

\subsection{Embedding of $M$ with uniform boundedness conditions on the geometry} 
In this section, we introduce a class of manifolds that admit an embedding into a Euclidean space with certain boundedness conditions on the metric projection. This concept allows us to extend vector fields on $M$ to the ambient space and control the Euclidean derivatives of these extensions. The definition is inspired by the notion of manifolds of bounded geometry, see~\cite[Definition~2.1]{eldering2013normally}. Here, we additionally assume the existence of a uniform tubular neighborhood of the normal bundle.

Let $N \in \N$ and $\iota: M \to \R^N$ be an isometric embedding of $M$ into $\R^N$. By \cite[Theorem~1]{leobacher2021existence}, there exists an open set $U \supset \R^N$ containing $\iota(M)$ such that for every $x \in U$ there exists a unique minimizer 
\begin{align} \label{eq:proj}
	\proj(x):= \argmin_{y \in \iota(M)} d(x,y)
\end{align}
that satisfies $x-\proj(x) \in (T_{\proj(x)} \iota(M))^\perp$ and $\proj\einschraenkung_U: U \to \R^N$ is $C^\infty$. Clearly, we have $\proj(x)=x$ for all $x \in \iota(M)$.
Moreover, for $x \in U$ let $\Proj_x \in \R^{N \times N}$ be the matrix that corresponds to the orthogonal projection from $\R^N$ onto the tangent space $T_{\proj(x)} \iota(M)$ of $\iota(M)$ at $\proj(x)$.

\begin{definition} \label{def:BG}
	Let $\alpha \in \N_0$. We say that a complete and connected $\mathcal C^\infty$-manifold $M$ is a $\mathrm{BG}(\alpha)$-manifold if there exist $N \in \N$, $r>0$, an isometric embedding $\iota:M \to \R^N$ and an open set $U\subset \R^N$ such that $x+v \in U$ for all $x \in \iota(M)$ and $v \in (T_x \iota(M))^\perp$ with $|v|\le r$ and $\proj:U \to \R^N$ and $\Proj:U \to \R^{N \times N}$ exist and their derivatives up to order $\alpha$ exist and are uniformly bounded.
\end{definition}

In the Euclidean setting, i.e. $M = \R^N$, one can set $\iota=\id$ and note that $(T_x \iota(M))^\perp=\emptyset$. Therefore, the Euclidean space is a $\mathrm{BG}(\alpha)$-manifold for all $\alpha \in \N_0$. Moreover, every compact $C^\infty$-manifold is a $\mathrm{BG}(\alpha)$-manifold for all $\alpha \in \N_0$, see e.g.~\cite{leobacher2021existence}.

The motivation for introducing $\mathrm{BG}(\alpha)$-manifolds is that we can extend functions on $M$ to functions defined on the ambient space $\R^N$ such that uniform boundedness of the Riemannian derivatives is equivalent to uniform boundedness of the Euclidean derivatives for the extended functions. For the definition of the spaces $\cC_b^\alpha(N)$ and $\mathfrak X_b^\alpha(N)$ for a manifold $N$ and $\alpha \in \N_0$ as well as the respective norms $\|\cdot\|_{\cC_b^\alpha(N)}$ and $\|\cdot\|_{\mathfrak X_b^\alpha(N)}$, see Definition~\ref{def:bounded}.

\begin{lemma} \label{lem:BG}
	Let $\alpha \in \N_0$ and $M$ be a $\mathrm{BG}(\alpha)$-manifold with corresponding embedding $\iota:M \to \R^N$, open set $U$ and constant $r>0$. Then
	\begin{enumerate}
		\item [(i)] there exists a constant $C_1 \ge 0$ such that for every function $g \in \mathcal C_b^{\alpha+1}(U)$ one has
		\begin{align*}
			\|g\einschraenkung_{\iota(M)}\|_{\cC^{\alpha+1}_b(\iota(M))} \le C_1 \|g\|_{\cC^{\alpha+1}_b(U)},
		\end{align*}
		\item[(ii)] there exists a constant $C_2 \ge 0$ such that for every $g \in \mathcal C_b^{\alpha}(\iota(M))$
		there exists an extension $\hat g \in \mathcal C_b^{\alpha}(U)$ of $g$ with
		\begin{align*}
			\|\hat g\|_{\cC^{\alpha}_b(U)} \le C_2 \|g\|_{\cC^{\alpha}_b(\iota(M))}
		\end{align*}
		and
		\item[(iii)] there exists a constant $C_3 \ge 0$ such that for every $V \in \mathfrak X^\alpha_b(\iota(M))$ there exists an extension $\hat V \in \mathfrak X_b^\alpha(\R^N)$ of $V$ with $\hat V(x)=0$ for all $x \in \R^N$ with $d(x,\iota(M)) \ge r/2$ and
		\begin{align*}
			\|\hat V\|_{\mathfrak X^{\alpha}_b(\R^N)} \le C_3 \|V\|_{\mathfrak X^{\alpha}_b(\iota(M))}.
		\end{align*}
	\end{enumerate}
\end{lemma}

\begin{proof}		
	(i): Denote $\tilde g=g \einschraenkung_{\iota(M)}$. Then $\grad \tilde g(x)=\Proj_x \, \grad g(x)$
	% we get $\|\tilde g\|_{\cC^{1}_b(\iota(M))} \le  \|g\|_{\cC^{1}_b(U)}$. 
	and, if $\alpha \ge 1$, for $v \in T_x \iota(M)$ we have
	\begin{align*}
		\Hess \tilde g(x)v &= \nabla_v \grad \tilde g= \Proj_x \Bigl( \frac{d}{dt}\einschraenkung_{t=0} \Proj_{\gamma_t} \, \grad g(\gamma_t)\Bigr) \\
		&= \Proj_x \Bigl( \Hess g(x) v + (D \Proj_{x}v) \grad g(x)  \Bigr).
	\end{align*}
	where $\gamma:\R\to \iota(M)$ is a smooth curve satisfying $\gamma_0=x$ and $\dot \gamma_0=v$.	
	%	Since $M$ is a $\mathrm{BG}(\alpha)$-manifold we have 
	%	$
	%	\kappa_1:= \sup_{x \in \iota(M)} \|D \Proj_x\|<\infty
	%	$
	%	and, using $\sup_{x \in \iota(M)}\|\Proj_x\|\le 1$, we get 
	%	$$
	%	\| \Hess \tilde g(x)v \| \le ( \|\Hess g(x)\| + \|D\Proj_x\| \, \|\grad g\| ) \|v\|\le (1+\kappa_1 ) \|g\|_{\cC^2_b(U)}  \|v\|.
	%	$$
	If $\alpha \ge 2$, we get for vector fields $V_1,V_2,V_3 \in \mathfrak X^2(\iota(M))$ that
	\begin{align*}
		\nabla^3 \tilde g (V_1,V_2,V_3) = V_1 ((\nabla^2 \tilde g) (V_2,V_3))-(\nabla^2 \tilde g)(\nabla_{V_1}V_2,V_3)-(\nabla^2 \tilde g)(V_2, \nabla_{V_1},V_3)
	\end{align*}
	as well as
	$
	V_1 ((\nabla^2 \tilde g) (V_2,V_3)) = \langle \nabla_{V_1} (\Hess \tilde g \, V_2), V_3\rangle + (\nabla^2 \tilde g)(V_2, \nabla_{V_1}V_3).
	$
	Now, for all $x \in \iota(M)$ and $v \in T_x \iota(M)$ we have $\Proj_x \, \Hess g(x) v = \Proj_x \, \Hess g(x) \Proj_x v$ so that
	\begin{align*}
		\nabla_{V_1(x)} (&\Hess \tilde g \,  V_2)= \Proj_x (D\Proj_x(V_1(x))) \Hess g(x) V_2(x)+ \Proj_x D^3g(x)(V_1(x), V_2(x))  \\
		& +\Proj_x \Hess g(x) (D\Proj_x (V_1(x))) V_2(x)+ \Proj_x \Hess g(x)\Proj_x (DV_2(x)(V_1(x)))  \\
		& + \Proj_x (D\Proj_x(V_1(x)))(D\Proj_x(V_2(x))) \grad g(x)+\Proj_x \, D^2 \Proj_x (V_1(x),V_2(x))\grad g(x)\\
		&+ \Proj_x (D\Proj_x(DV_2(x)(V_1(x))))\grad g(x) + \Proj_x (D\Proj_x (V_2(x))) \Hess g(x) V_1(x).
	\end{align*}
	Using that 
	\begin{align*}
		(\nabla^2 \tilde g)(\nabla_{V_1}V_2,V_3)(x)= \langle &\Proj_x\Hess g(x) \Proj_x (DV_2(x)(V_1(x)))\\
		&+ \Proj_x (D\Proj_x(\Proj_x (DV_2(x)(V_1(x)))))\grad g(x), V_3(x) \rangle 
	\end{align*}
	and $D\Proj_x v=0$ for all $v \in (T_x \iota(M))^\perp$, we get 
	\begin{align*}
		\nabla^3 \tilde g(V_1,&V_2,V_3)(x) = \langle \Proj_x (D\Proj_x(V_1(x))) \Hess g(x) V_2(x)+ \Proj_x D^3g(x)(V_1(x), V_2(x))  \\
		& +\Proj_x \Hess g(x) (D\Proj_x (V_1(x))) V_2(x) + \Proj_x (D\Proj_x(V_1(x)))(D\Proj_x(V_2(x))) \grad g(x)  \\
		& +\Proj_x \, D^2 \Proj_x (V_1(x),V_2(x))\grad g(x)+ \Proj_x (D\Proj_x (V_2(x))) \Hess g(x) V_1(x), V_3(x) \rangle.
	\end{align*}
	%	Thus, for $\kappa_2 := \sup_{x \in \iota(M)} \|D^2 \Proj_x\|<\infty$ it holds that
	%	\begin{align*}
		%		\sup_{x \in \iota(M)}\frac{\|\nabla^3 \tilde g(V_1,V_2,V_3)(x)\|}{\|V_1(x)\|\, \|V_2(x)\| \, \|V_3(x)\|} & \le \sup_{x \in \iota(M)} \bigl(3 \|D\Proj_x\| \, \|\Hess g(x)\|+ \|D^3g(x)\| \\
		%		&\qquad \quad  + \|D\Proj_x\|^2 \|\grad g(x)\|+\|D^2\Proj_x\| \, \|\grad g(x)\|\bigr) \\
		%		& \le (3\kappa_1+1+\kappa_1^2+\kappa_2) \|g\|_{\cC^3_b(U)}.
		%	\end{align*}
	Using $\sup_{x \in \iota(M)} \max_{\beta=0, \dots, \alpha} \|D^\beta P_x\|<\infty$, we proved (i) for $\alpha \in \{0,1,2\}$. The proof for higher derivatives is analogous.
	
	(ii): For $x \in U$ define
	$
	\hat g(x) := g(\proj(x))
	$.
	Then $\hat g \in \mathcal C^\alpha(U)$ and $\|\hat g(x)\|_{\cC^0_b(U)} =  \|g(x)\|_{\cC^0_b(\iota(M))}$.
	If $\alpha \ge 1$, we have for $x \in U$ and $i=1, \dots, N$
	\begin{align*}
		\frac{d}{dx_i} \hat g(x) = \langle \grad g(\proj(x)), D\proj(x) e_i\rangle ,
	\end{align*}
	where $(e_1,\dots, e_N)$ denotes the standard orthonormal basis of $\R^N$. 
	%	Thus,
	%	\begin{align*}
		%	\sup_{x \in U} \Bigl|\frac{d}{dx_i} \hat g(x)\Bigr| \le   \sup_{x \in U}\|D\proj(x)\| \, \|g\|_{\cC^1_b(\iota(M))}.
		%	\end{align*}
	If $\alpha \ge 2$, note that for all $x \in \iota(M)$ we have $\grad \hat g(x)=\grad g(x) \in T_x \iota(M)$. Thus, for all $i, j=1, \dots, N$ and $x \in U$
	\begin{align*}
		\frac{d}{dx_j} \frac{d}{dx_i} &\hat g (x) = \frac{d}{dx_j} \langle \grad \hat g(\proj(x)), D\proj(x) e_i\rangle \\
		&= \langle \Proj_x \frac{d}{dx_j} \grad \hat g(\proj(x)), D\proj(x) e_i\rangle + \langle  \grad \hat g(\proj(x)),  \frac{d}{dx_j} D\proj(x) e_i\rangle \\
		&= \langle \nabla_v \grad  g, D\proj(x) e_i\rangle + \langle  \grad g(\proj(x)),  D^2\proj(x) (e_i,e_j)\rangle,
	\end{align*}
	where $v:= D\proj(x)e_j \in T_{\proj(x)} \iota(M)$. 
	%	Hence,	
	%	\begin{align*}
		%		\sup_{x \in U} \Bigl| \frac{d}{dx_j} \frac{d}{dx_i} \hat g(x) \Bigr| \le  \Bigl( \sup_{x \in U}\|D\proj(x)\|^2 +  \sup_{x \in U}\|D^2\proj(x)\|\Bigr) \|g(x)\|_{\cC^2_b(\iota(M))}.
		%	\end{align*}
	Using $\sup_{x \in \iota(M)} \max_{\beta=0, \dots, \alpha} \|D^\beta \proj(x)\|<\infty$, we proved (i) for $\alpha \in \{0,1,2\}$. The proof for higher derivatives is analogous.

	(iii): 
	Let $\psi:[0,\infty)\to [0,1]$ be a $C^\infty$-cutoff function that satisfies $\psi(0)=1$, $\psi(y)=0$ for all $y\ge r^2/4$ and such that all derivative of $\psi$ at $0$ and $r^2/4$ vanish.
	For $x \in \R^N$ define
	\begin{align*}
		\hat V(x) :=\begin{cases}
			\psi(|\proj(x)-x|^2) V(\proj(x)) ,& \text{ if } x \in U \\
			0,& \text{ otherwise.}
		\end{cases}
	\end{align*}
	Then, $\hat V \in \mathfrak X^\alpha(U)$, $\|\hat V(x)\|_{\mathfrak X^0_b(M)} =  \|V(x)\|_{\mathfrak X^0_b(\iota(M))}$ and $\hat V=0$ for all $x \in \R^N$ with $d(x,\iota(M)) \ge r/2$.
	If $\alpha = 1$, we use $v = P_x v$ for $x \in \iota(M)$ and $v \in T_x \iota(M)$ and get for $i=1,\dots, N$ and $x \in U$
	\begin{align*}
		\frac{d}{dx_i} \hat V(x) = &2 \psi'(|\proj(x)-x|^2)(\proj(x)-x)^{\dagger}\bigl(\frac{d}{dx_i}\proj(x)-e_i \bigr) V(\proj(x))  \\
		&+ \psi(|\proj(x)-x|^2)\Bigl(\frac{d}{dx_i}\Proj_{x}\Bigr) V(\proj(x))\\
		& +  \psi(|\proj(x)-x|^2)\Proj_{x} (DV(\proj(x)))(D\proj(x) e_i),
	\end{align*}
	where $a^\dagger$ denotes the transpose of a vector $a\in \R^N$ and $(e_1,\dots, e_N)$ denotes the standard orthonormal basis of $\R^N$.
	Using $\Proj_{x} (DV(\proj(x)))(D\proj(x) e_i)= \nabla_{D \proj(x) e_i}V$ and the bounds for $DP_x$ and $D\proj(x)$ we get the statement for $\alpha=1$.
	%	, we have
	%	\begin{align*}
		%		\sup_{x \in U}\Bigl| \frac{d}{dx_i} \hat V(x)\Bigr| \le \|V(x)\|_{\mathfrak X^1_b(\iota(M))} \Bigl(\Bigl(2 \sup_{y \in [0,r^2]} \psi'(y) r+1\Bigr) \Bigl( \sup_{x \in U} \|D\proj(x)\|+1 \Bigr) + \sup_{x \in U} \| D\Proj_x\| \Bigr).
		%	\end{align*}
	The proof for higher derivatives is analogous.
\end{proof}

Using the extensions of the vector fields and the test functions constructed in Lemma~\ref{lem:BG}, the assumptions in Definition~\ref{def:reg} can be verified for $M$ replaced by $\R^N$. A completeness result for the SDE \eqref{eq:SDEdefi} on $\R^N$ with Lipschitz continuous coefficients can be found in~\cite[Theorem~7.5]{da2014stochastic}. Boundedness of the first two derivatives of $x \mapsto \tilde E[g(X_t(x))]$ can be found in~\cite[Theorem~9.23]{da2014stochastic} and~\cite[Remark~9.4]{da2014stochastic}.

\begin{proposition} \label{prop:regularity}
	Let $\alpha \in \{2,3, \dots\}$ and $M$ be a $\mathrm{BG}(\alpha+1)$-manifold.
	Then $M$ has regularity $\alpha$.
\end{proposition}

\section{Order $2$ Approximation} \label{sec:order2}
In this section, we prove the second main result which quantifies the order of approximation of RSGD by the solution to the RSMF (\ref{eq:SDEintro}) with coefficients $G$ and $B$ as in (\ref{eq:Gintro}) and (\ref{eq:Bintro}), respectively.

\begin{theorem} \label{theo:main2}
	Let $\grad \tilde f \in \tilde {\mathfrak X}_b^5(M)$, $\retr: TM \to M$ be a uniform second order retraction and assume that $M$ has regularity $4$, see Definition~\ref{def:reg}. Moreover, assume that for all $x \in M$ one has $\E_\vartheta[\grad \tilde f(x,\xi)]=\grad f(x)$ and 
	\begin{align} \label{eq:2836289656}
		\mathfrak f := \sup_{x \in M} \int_\Xi  \| \grad \tilde f(x,\xi) \|^3 \, \vartheta(d\xi) < \infty.
	\end{align}
	Then, for every $g \in \mathcal C^4_b(M)$ and $T\ge 0$, there exists a constant $C \ge 0$ such that for all $\eta \ge 0$
	\begin{align*}
		\sup_{x \in M} \sup_{n=0, \dots, \lfloor T/\eta \rfloor }|\E[g(Z_n^\eta(x))]-\tilde \E[g(X_{n\eta}^\eta(x))] | \le C \eta^2.
	\end{align*}
\end{theorem}

The proof consists in comparing an evolution step of RSGD with the solution of the RSMF for time $\eta$ on test functions of the form $\Psi_t$, for $t \ge 0$, defined in Definition~\ref{def:reg}. First, we show that the coefficients of the RSMF are sufficiently regular.

\begin{lemma} \label{lem:coefficients}
	Let $T>0$, $\alpha \in \N$, $\grad \tilde f \in \tilde{\mathfrak X}^\alpha_b(M)$ and $B^\eta, G^\eta$ as in \eqref{eq:Gintro} and \eqref{eq:Bintro}. Then, for all $0 \le \eta \le T$ it holds that $B^\eta \in \mathfrak X^{\alpha-1}_b(M)$ and $G^\eta \in \tilde{\mathfrak X}^\alpha_b(M)$ and
	\begin{align*}
		\sup_{0 \le \eta \le T} \|B^\eta\|_{\mathfrak X^{\alpha-1}_b(M)} \vee \sup_{0 \le \eta \le T} \|G^\eta\|_{\tilde {\mathfrak X}^\alpha_b(M)} < \infty.
	\end{align*}
\end{lemma}

\begin{proof}
	First, we will show that $\grad f \in \mathfrak X^\alpha_b (M)$. Let $\iota:M \to \R^N$ be an isometric embedding and $\hat f\in \tilde {\mathfrak X}^\alpha(\R^N)$ be an extension of $\grad \tilde f$. Then, for all compact sets $K \subset \R^N$ and $0 \le \beta \le \alpha$ we have that 
	\begin{align*}
		\int_\Xi \sup_{y \in K}\|D^\beta \hat f(y,\xi)\|^2 \, \vartheta(d\xi)<\infty.
	\end{align*}
	Using the dominated convergence theorem, we get $ \E_\vartheta[\hat f(\cdot,\xi)] \in \mathfrak X^\alpha(\R^N)$ and, thus, $\grad f \in \mathfrak X^\alpha(M)$. Moreover, for all $x \in M$ and $0\le \beta \le \alpha$,
	\begin{align} \label{eq:7723523}
		\|\nabla^\beta \grad f(x)\|^2 = \|\E_\vartheta[\nabla^\beta \grad \tilde f(x,\xi)]\|^2 \le  \|\nabla^\beta \grad \tilde f(x,\cdot)\|_\vartheta^2
	\end{align}
	where the right-hand side of (\ref{eq:7723523}) is uniformly bounded in $x$, since $\grad \tilde f \in \tilde {\mathfrak X}^\alpha_b(M)$. Thus, $\grad f \in \mathfrak X^\alpha_b (M)$ which immediately implies that
	$\nabla_{\grad f(\cdot)}(\grad f) \in \mathfrak X^{\alpha-1}_b(M)$ and $G^\eta \in \tilde{\mathfrak X}^\alpha_b(M)$.
	Similarly, we have for an extension $\hat G^\eta \in \tilde{\mathfrak X}^\alpha(\R^N)$ of $G^\eta$, $0\le \beta \le \alpha-1$ and all compact sets $K \subset \R^N$ that
	\begin{align*}
		\int_\Xi \sup_{y \in K}\|D^\beta (D{\hat G^\eta(y,\xi)}\hat G^\eta(y, \xi))\| \, \vartheta(d\xi)<\infty.
	\end{align*}
	Moreover, $z \mapsto \E_\vartheta[\|\nabla^\beta (\nabla_{G^\eta(z,\xi)}G^\eta(\cdot, \xi))\|]$ is uniformly bounded over $z\in M$ and $0 \le \eta \le T$ for all $0 \le \beta \le \alpha-1$,
	so that $z \mapsto \E_\vartheta[\nabla_{G^\eta(z,\xi)}G^\eta(\cdot, \xi)] \in {\mathfrak X}^{\alpha-1}_b(M)$.
\end{proof}

We next control the weak error of the diffusion approximation for one evolution step of RSGD.

\begin{lemma} \label{lem:SDEonestep}
	Let $\retr$ be a uniform second order retraction, $\grad \tilde f \in \tilde {\mathfrak X}_b^5(M)$ and $g \in \mathcal C^4_b(M)$ and $T\ge 0$. Moreover, assume that $M$ has regularity $4$ and \eqref{eq:2836289656} holds.
	Then, there exists a constant $C \ge 0$ that only depends on $T$, $\|g\|_{\mathcal C^4_b(M)}$, $\|\grad \tilde f\|_{\tilde{\mathfrak X}^5_b(M)}$ and $\mathfrak f$ such that for all $0 \le \eta \le T$
	\begin{align*}
		\sup_{x \in M}|\E_\vartheta[g(Z_1^\eta(x))]-\tilde \E[g(X_\eta^\eta(x))] | \le C \eta^3.
	\end{align*}
\end{lemma}

\begin{proof}
	Throughout this proof, $C$ denotes a constant that only depends on $T$, $\|g\|_{\mathcal C^4_b(M)}$, $\|\grad \tilde f\|_{\tilde{\mathfrak X}^5_b(M)}$ and $\mathfrak f$. Note that, using Lemma~\ref{lem:coefficients}, the norms $\|B^\eta\|_{\mathfrak X^4_b(M)}$ and $\|G^\eta\|_{\tilde{\mathfrak X}^5_b(M)}$ can be uniformly bounded over $0 \le \eta \le T$.
	
	Let $x \in M$ and $0 \le \eta \le T$. First, we consider the solution to the RSMF (\ref{eq:SDEintro}) and briefly write $(X_t)_{0 \le t \le \eta}:=(X_t^\eta(x))_{0 \le t \le \eta}$.
	By definition, we have for all $0 \le t \le \eta$
	\begin{align}
		\begin{split}
			\label{eq:2734862}
			g(X_t)-g(x) = &\int_0^t Bg(X_s) \, \rd s + \sqrt \eta   \int_0^t\int_\Xi \bar Gg(X_s,\xi) \, W(ds,d\xi) \\
			&+ \frac { \eta}{2} \int_0^t \int_\Xi \bar G\bar Gg(X_s,\xi) \vartheta(d\xi) \, ds.
		\end{split}
	\end{align}
	
	Now, for all $0 \le t \le \eta$ we get by Itô's isometry that
	\begin{align*}
		\Bigl[ \int_0^\cdot \int_\Xi \bar Gg(X_s,\xi) \, W(ds,d\xi)\Bigr]_t =  \int_0^t \| \bar Gg(X_s,\cdot)\|^2_\vartheta \, ds,
	\end{align*}
	so that 
	\begin{align*}
		(M_t)_{0 \le t\le \eta}:=\Bigl( \int_0^t \int_\Xi \bar Gg(X_s,\xi) \, W(ds,d\xi) \Bigr)_{0 \le t \le \eta}
	\end{align*}
	is a martingale and we have $\tilde \E[M_\eta]=0$.
	Moreover, for the first term on the right-hand side of (\ref{eq:2734862})  we get, using Itô's formula, 
	\begin{align*}
		Bg(X_s) = &Bg(x) + \int_0^s BB g(X_u) \, \rd u + \int_0^s \int_\Xi \sqrt{\eta} \bar GBg(X_u,\xi) \,  W(du,d\xi) \\
		&+\frac { \eta}{2} \int_0^s \int_\Xi \bar G\bar GBg(X_u,\xi) \,  \vartheta(d\xi) \, du.
	\end{align*}
	%	Now, 
	%	\begin{align*}
		%	\Bigl(\int_0^s \int_\Xi \sqrt{\eta} \bar GBg(X_u,\xi)\,  W(du,d\xi) \Bigr)_{0 \le s \le \eta}
		%	\end{align*}
	%	is a martingale and b
	Using Fubini's theorem, we get
	\begin{align*}
		\tilde \E\Bigl[ \int_0^\eta &\int_0^s \int_\Xi \sqrt{\eta} \bar GBg(X_u,\xi) \,  W(du,d\xi) \, ds \Bigr] \\
		&= \int_0^\eta \tilde \E \Bigl[ \int_0^s \int_\Xi \sqrt{\eta} \bar GBg(X_u,\xi) \,  W(du,d\xi) \Bigr] \, \rd s =0.
	\end{align*}
	We can bound
	\begin{align*}
		\tilde \E\Bigl[ \Bigl| \int_0^\eta \frac { \eta}{2} \int_0^s \int_\Xi \bar G\bar GBg(X_u,\xi) \, \vartheta(d\xi) \, du  \Bigr| \Bigr] \le \frac{\eta^3}{4} \sup_{x \in M} \E_\vartheta [|\bar G\bar GBg(x,\xi)|]\le C \eta^3,
	\end{align*}
	for a constant $C \ge 0$.
	Again, using Itô's formula
	\begin{align*}
		BBg(X_s) = &BBg(x) + \int_0^s BBB g(X_u) \, \rd u + \int_0^s \int_\Xi \sqrt{\eta} \bar GBBg(X_u,\xi)  \,  W(du,d\xi) \\
		&+\frac { \eta}{2} \int_0^s \int_\Xi \bar G\bar GBBg(X_u,\xi) \, \vartheta(d\xi) \, du,
	\end{align*}
	where
	%	, again, for a constant $C \ge 0$,
	%	$
	%	\sup_{x \in M}|BBBg(x)|\le C,
	%	$
	%	$
	%	\sup_{x \in M} \E_\vartheta [|\bar G\bar GBBg(x,\xi)|]\le C
	%	$
	%	and 
	for all $s \in [0,\eta]$ one has
	$
	\tilde \E \bigl[ \int_0^s \int_\Xi \sqrt{\eta} \bar GBBg((X_u,\xi)) \,  W(du,d\xi) \bigr]=0.
	$
	
	We bound the contribution of the last term on the right-hand side of (\ref{eq:2734862}). An application of the dominated convergence theorem gives $\varphi(x)= \int_\Xi \bar G\bar Gg(x,\xi) \, \vartheta(d\xi) \in \mathcal C^2_b(M)$. Thus, 
	\begin{align*}
		\varphi(X_s)-\varphi(x) = &\int_0^t B\varphi(X_s) \, \rd s + \sqrt \eta   \int_0^t\int_\Xi \bar G\varphi (X_s,\xi) \, W(ds,d\xi)\\
		& + \frac { \eta}{2} \int_0^t \int_\Xi \bar G\bar G\varphi(X_s,\xi) \vartheta(d\xi) \, ds, 
	\end{align*}
	where, for a constant $C\ge 0$,
	%	$\sup_{x \in M}|B\varphi(x)|\le C,$
	%	$
	%	\sup_{x \in M}\E_\vartheta[|\bar G\bar G\varphi(x,\xi)|]\le C
	%	$
	%	and 
	for all $0 \le u \le \eta$ one has 
	$
	\tilde \E\bigl[\int_0^u\int_\Xi \bar G\varphi (X_s,\xi) \, W(ds,d\xi)\bigr]=0.
	$
	
	Altogether, we can bound the terms of order $\cO(\eta^3)$ in the equations above and get a constant $C \ge 0$ such that 
	\begin{align*} %\label{eq:2977734}
		\sup_{x \in M} \,  \Bigl|\tilde \E[g(X_\eta^\eta(x))]-g(x)-\eta Bg(x)-\frac{1}{2} \eta^2 \Bigl(BBg(x)+ \int_\Xi \bar G\bar G g(x,\xi) \, \vartheta(d\xi) \Bigr)\Bigr| \le C \eta^3.
	\end{align*}
	
	Recall that, for all $x \in M$ and $\xi \in \Xi$, $BB g(x)= \langle (\Hess g(x))B(x),B(x) \rangle + (\nabla_{B(x)}B)g$ and $\bar G\bar G g(x,\xi )=\langle \Hess g(x) \bar G(x,\xi), \bar G(x,\xi) \rangle + (\nabla_{\bar G(x,\xi)}\bar G(\cdot, \xi))g$.
	%	Therefore, there exists a constant $C \ge 0$ such that for all $x \in M$
	%	\begin{align}\begin{split} \label{eq:188822}
			%			\Bigl| &\eta Bg(x)+\frac{1}{2} \eta^2 \Bigl(BBg(x)+\int_\Xi  \bar G\bar G g(x,\xi) \, \vartheta(d\xi) \Bigr) +\eta \langle \grad f(x), \grad g(x)\rangle \\
			%			& - \frac 12 \eta^2 \Bigl(\langle (\Hess g(x))\grad f (x), \grad f(x) \rangle +\int_\Xi \langle \Hess g(x) \bar G(x,\xi), \bar G(x,\xi) \rangle \, \vartheta(d\xi) \Bigl)\Bigr|\le C \eta^3.
			%		\end{split}
		%	\end{align}
	Hence, using the definition of $B$ in \eqref{eq:Bintro} there exists a constant $C\ge 0$ such that
	\begin{align} \begin{split} \label{eq:SDEonestep}
			\Bigl|\tilde \E[g(X_\eta)]-g(x)+\eta \langle \grad f(x), &\grad g(x)\rangle -\frac{1}{2} \eta^2 \Bigl(\langle (\Hess g(x))\grad f (x), \grad f(x) \rangle \\
			&- \int_\Xi \langle \Hess g(x) \bar G(x,\xi), \bar G(x,\xi) \rangle \, \vartheta(d\xi) \Bigr) \Bigr| \le C\eta^3.
		\end{split}
	\end{align}
	
	Next, we turn to the RSGD scheme.
	For $x \in M$ and $\xi \in \Xi$, let $(\gamma_t)_{t \in [0,1]}=(\retr_{x}(-\eta t \,  \grad \tilde f(x,\xi)))_{t \in [0,1]}$ and note that
	\begin{align*}
		g(\gamma_1)&= g(x)+ \frac{d}{ds}\einschraenkung_{s=0} g(\gamma_s) + \frac 12 \frac{d^2}{ds^2}\einschraenkung_{s=0} g(\gamma_s)+\int_0^1 \int_0^s \int_0^u \frac{d^3}{d\ell^3} g(\gamma_\ell) \, d\ell \, du\, ds.
	\end{align*}	
	Since $\retr: TM \to M$ is a second order retraction, see Definition~\ref{def:retraction}, we get 
	$
	\frac{d}{ds}\einschraenkung_{s=0} g(\gamma_s) = - \eta \langle \grad g(x) , \grad \tilde f(x,\xi)\rangle
	$
	as well as
	\begin{align*}
		\frac{d^2}{ds^2}\einschraenkung_{s=0} g(\gamma_s) &= \langle \nabla_{\dot \gamma_0} \grad g, \dot \gamma_0 \rangle  + \langle \grad g(\gamma_0), \frac{\nabla}{ds}\einschraenkung_{s=0}\dot \gamma_s \rangle \\
		&= \eta^2 \langle (\Hess g(x))(\grad \tilde f(x,\xi)), \grad \tilde f(x,\xi) \rangle.
	\end{align*}
	For the remainder, note that for all $\ell \in [0,1]$
	% and denote by $\Upsilon^\ell \in \mathfrak X^2(M)$ a vector field that satisfies $\dot \gamma_w = \Upsilon^\ell(\gamma_w)$ for all $w$ in a small neighborhood of $\ell$. Thus,
	\begin{align*}
		\frac{d^3}{d\ell^3}g(\gamma_\ell) 
		&= \nabla^3g (\dot \gamma_\ell, \dot \gamma_\ell, \dot \gamma_\ell) + 3 \nabla^2g\Bigl(\frac{\nabla}{d\ell} \dot \gamma_\ell, \dot \gamma_\ell\Bigr) + \langle \grad g(\gamma_\ell), \frac{\nabla^2}{d\ell^2} \dot \gamma_\ell \rangle,
	\end{align*}
	where we used that the Riemannian Hessian is symmetric.
	Since $\retr$ is a uniform second order retraction, there exists a constant $C\ge 0$ that only depends on $\|g\|_{\mathcal C^3_b(M)}$ such that 
	$
	\bigl| \frac{d^3}{d\ell^3}g(\gamma_\ell) \bigr| \le C \eta^3 \|\grad \tilde f(x,\xi)\|^3.
	$
	Taking expectation, 
	\begin{align}\begin{split} \label{eq:9763633}
			\bigl|\E_\vartheta[g(&\retr_{x}(- \eta  \, \grad \tilde f(x,\xi)))]-g(x)+\eta \langle \grad f(x),\grad  g(x)\rangle\\
			&-\frac 12\eta^2 \E_\vartheta[ \langle (\Hess g(x))(\grad \tilde f(x,\xi)), \grad \tilde f(x,\xi) \rangle] \bigr|
			\le  C \eta^3 \E_\vartheta [  \| \grad \tilde f(x,\xi) \|^3 ],
		\end{split}
	\end{align}
	for a constant $C\ge 0$, where
	\begin{align*}
		\E_\vartheta[\langle (\Hess g(x)) \grad \tilde f(x,\xi), (\grad \tilde f(x,\xi)) \rangle]= & \langle (\Hess g(x)) \grad f(x),  (\grad f(x)) \rangle \\
		&+ \E_\vartheta[\langle (\Hess g(x)) \bar G(x,\xi),  \bar  G(x,\xi) \rangle].
	\end{align*}
	Using that $\E_\vartheta[g(\retr_{x}(- \eta \,  \grad \tilde f(x,\xi)))]= \E[g(Z_1^\eta(x))]$ and comparing (\ref{eq:9763633}) with (\ref{eq:SDEonestep}), we get for a constant $C \ge 0$ that
	\begin{align*}
		\sup_{x \in M}|\E[g(Z_1^\eta(x))]-\tilde \E[g(X_\eta^\eta(x))] | \le C \eta^3.
	\end{align*}
\end{proof}

\begin{proof}[Proof of Theorem~\ref{theo:main2}]
	Fix $T \ge 0$. By Lemma~\ref{lem:coefficients}, we have $
	B^\eta \in \mathfrak X_b^4(M)$ and $G^\eta \in \tilde {\mathfrak X}^5_b(M)$ uniformly over $0 \le \eta \le T$.
	By Definition~\ref{def:reg}, for all $0 \le \eta \le T$ there exists a unique solution to the SDE (\ref{eq:SDEintro}) that does not explode in finite time and for all $g \in \cC_b^4$ there exists a constant $C \ge 0$ such that for all $0 \le t \le T$ and $0 \le \eta \le T$ the function $\Psi_t^\eta: M \to \R$ given by $\Psi_t^\eta(x)=\tilde \E[g(X_t^\eta(x))]$ satisfies
	$
	\|\Psi_t^\eta\|_{\cC^4_b(M)} \le C.
	$
	
	We conceive the probability spaces $(\Omega, \cF,\P)$ and $(\tilde \Omega, \tilde \cF, \tilde \P)$ as the projections of the product space $(\Omega \times \tilde \Omega, \cF \otimes \tilde \cF, \P \times \tilde \P)$ so that $(Z_{n}^\eta(x))_{n \in \N_0}$ and $(X_t^\eta(x))_{t \ge 0}$ are independent processes. 
	For $x \in M$ and $n \in \N$ we get, using the triangle inequality and the tower property of the conditional expectation,
	\begin{align*}
		\begin{split} %\label{eq:82365232}
			|\E[g(Z_n^\eta(x))]- \tilde \E&[g(X^\eta_{n \eta}(x))]|  = \Bigl| \sum_{i=1}^n \E\Bigl[\Psi^\eta_{(n-i)\eta}(Z_i^\eta(x))- \tilde \E \bigl[\Psi^\eta_{(n-i)\eta}(X^\eta_\eta(Z_{i-1}^\eta(x))) \bigr]\Bigr] \Bigr| \\
			%	&\le \sum_{i=1}^n \Bigl|\E\Bigl[ \E\bigl[\Psi^\eta_{(n-i)\eta}(Z_i^\eta(x)) | \mathcal F_{i-1}\bigr] -  \tilde \E \bigl[ \Psi^\eta_{(n-i)\eta}(X^\eta_\eta(Z_{i-1}^\eta(x))) \bigr]\Bigr]\Bigr|\\
			& \le \sum_{i=1}^n \sup_{z \in M} \bigl| \E\bigl[\Psi^\eta_{(n-i)\eta}(Z_1^\eta(z))] -  \tilde \E \bigl[ \Psi^\eta_{(n-i)\eta}(X^\eta_\eta(z)) \bigr]\bigr|.
		\end{split}
	\end{align*}
	Applying Lemma~\ref{lem:SDEonestep} to each summand on the right-hand side of the inequality above, there exists a constant $C\ge 0$ such that for all $n \in \{0, \dots, \lfloor T/\eta \rfloor\}$
	\begin{align*}
		|\E[g(Z_n^\eta(x))]- \tilde \E[g(X^\eta_{n \eta}(x))]| \le n C \eta^3 \le C T \eta^2 .
	\end{align*}
\end{proof}

\section{Examples} \label{sec:exa}

\subsection{Principal component analysis} \label{sec:PCA}
In principal component analysis (PCA), the aim is to find the $r$ principal eigenvectors of a matrix $A := \E_\vartheta[z(\xi) z(\xi)^T]$, where $z:\Xi \to \R^n$ is a random data vector and $v^T$ denotes the transpose of a vector or a matrix $v$. 
For this problem, a natural choice of the search space is the Stiefel manifold $\mathrm{St}(r,n)$ \cite{O1992, mahony1996gradient} or the Grassmann manifold $G(r,n)$ \cite{HHT2007}.
The Stiefel manifold $\mathrm{St}(r,n)$ given by
$
\mathrm{St}(r,n) := \{ B \in \R^{n\times r}: B^T B = \1_r \}
$
is a compact, smooth manifold of dimension $nr - \frac 12 r(r+1)$, see \cite[Section~3.3.2]{absil2009optimization}.
The Grassmann manifold $G(r,n)$ consisting of all $r$-dimensional subspaces of $\R^n$ is a compact, smooth manifold of dimension $r (n-r)$, see \cite[Section~3.4.4]{absil2009optimization}. By compactness, $\mathrm{St}(r,n)$ and $G(r,n)$ are $\mathrm{BG}(\alpha)$-manifolds for all $\alpha \in \N$, see Definition~\ref{def:BG}.

On $\mathrm{St}(r,n)$ we can define a loss function 
\begin{align*}
	f(B) = - \frac 12 \operatorname{tr}(B^T A B),
\end{align*}
which is minimal if $B$ consists of eigenvectors that correspond to the $r$ largest eigenvalues of $A$. We can regard $f$ as a function on the Euclidean space $\R^{n\times r}$ with gradient
$
Df(B) = - A B.
$
If $\mathrm{St}(r,n)$ is equipped with the Riemannian submanifold metric inherited from $\R^{n \times r}$ then the Riemannian gradient of $f$ at $B \in \mathrm{St}(r,n)$ is given by $\grad f(B) = \Proj_B (Df(B))$, where
\begin{align*}
	\Proj_B (Z) = \frac 12 Z (B^TZ-Z^T B) + (\1_n-BB^T)Z
\end{align*}
is the orthogonal projection of $Z \in \R^{n \times r}$ onto the tangent space $T_B \mathrm{St}(r,n) = \{ \Delta \in \R^{n \times r}: \Delta^T B+B^T \Delta=0 \}$, see~\cite{edelman1999geometry}, and $\1_n$ denotes the identity matrix in $\R^{n \times n}$. By \cite[Proposition~3.4]{absil2012projection}, for all $B \in \mathrm{St}(r,n)$ and $\bar B \in \R^{n \times r}$ with $\|B-\bar B\|< 1$ the projection (\ref{eq:metricpro2}) exists, is unique and has the expression
\begin{align*}
	\proj(\bar B) = UV^T.
\end{align*}
Here, for a matrix $X \in \R^{n\times r}$, $\|X\|$ denotes the Frobenius norm, given by
$
\|X\|^2 = \operatorname{tr}(X^T X)
$
and $U\in \R^{n \times n}$, $V\in \R^{r \times r}$ are orthogonal matrices given by the singular value decomposition
$
B = U \Sigma V^T,
$
where
$\Sigma\in \R^{n \times r}$ denotes the diagonal matrix consisting of the singular values of $B$ in decreasing order.
Using Lemma~\ref{lem:secondorderretraction}, we get a uniform second order retraction via
$
\retr_B( \Delta) = \proj(B+ c(B,\Delta))
$,
where $c: \R^{n\times r} \times \R^{n \times r}\to \R^{n\times r}$ is chosen according to Lemma~\ref{lem:secondorderretraction}.

See also \cite{edelman1999geometry, liu2019quadratic, lin2020projection} for different retractions and Riemannian metrics on $\mathrm{St}(r,n)$ as well as an analysis of the geometry of $G(r,n)$. For other common optimization tasks on matrix manifolds we refer the reader to Section~2 in~\cite{absil2009optimization}.

\subsection{Normalizing the weights of a neural network} \label{sec:weightnorm}
We consider optimizing a neural network with a positive homogeneous activation function. 
For simplicity, we restrict the section to neural networks with one-hidden layer although the arguments remain true for deep neural networks with multiple hidden layers. 

Let $d_0,d_1 \in \N$ and $\sigma:\R \to \R$ be positive homogeneous, i.e. $\sigma(\lambda x) = \lambda \sigma(x)$ for all $x \in \R$ and $\lambda \ge 0$. This property is satisfied by e.g. the ReLU activation function or any linear activation. 
The networks configuration can be described by the weights  $W^1=(w_{j,i}^1)_{j=1,\dots,d_1,i=1,\dots,d_0} \in \R^{d_1\times d_0}$ and $W^2=(w_{1}^2,\dots, w_{d_1}^2) \in \R^{d_1}$, and the biases $b^1=(b^1_i)_{i=1,\dots,d_1}\in \R^{d_1}$ and $b^2 \in \R$.
For $j=1, \dots,d$, we write $w_j^1=(w_{j,1}^1, \dots, w_{j,d_{\mathrm{in}}}^1)^T$. The search space is given by
\begin{align*}
	\IW=(W^1,b^1, W^2, b^2) \in \R^{d_1\times d_{0}}\times \R^{d_1}\times \R^{d_1}\times \R 
	=:\cW_{d_0,d_1}
\end{align*}
and for $\IW \in \cW_{d_0,d_1}$ we define $\mathfrak N^\IW: \R^{d_{0}} \to \R$ via
\begin{align*}
	\mathfrak N^\IW(x)= \sum_{j=1}^{d_1} w_{j}^2 \,  \sigma\bigl( x^T w_{j}^1+b_j^1\bigr)+b^2,
\end{align*}
which is the \emph{response} of the neural network to the input $x$ for the configuration $\IW$.
Using the positive homogeneity of $\sigma$ we have 
\begin{align*}
	\mathfrak N^\IW(x)= \sum_{j=1}^{d_1} w_{j}^2 \|w_j^1\| \,  \sigma\Bigl( x^T \frac{w_{j}^1}{\|w_j^1\|}+\frac{b_j^1}{\|w_j^1\|}\Bigr)+b^2
\end{align*}
so that the optimization can be restricted to the Riemannian manifold
\begin{align*}
	M:= \{\IW \in \cW_{d_0,d_1}: \|w_j^1\|=1 \, \forall j=1, \dots, d_1 \} \cong (\IS^{d_0-1})^{d_1} \times \R^{d_1}\times \R^{d_1}\times \R .
\end{align*}
Although $M$ is non-compact, it is clearly a $\mathrm{BG}(\alpha)$-manifold for all $\alpha \in \N$. A uniform second order retraction is given by a componentwise stereographic projection, see Example~\ref{exa:secondretraction}, or componentwise metric projection, see (\ref{eq:metricpro2}). The idea of decoupling the length and the direction of the weight vectors in neural networks was popularized by Salimans and Kingma with their weight normalization algorithm \cite{salimans2016weight}. The corresponding gradient flow was shown to have a beneficial implicit bias while being less sensitive to the initialization \cite{poggio2020complexity, morwani2022inductive, chou2023robust}.

\subsection{Hyperbolic space} \label{sec:hyperbolic}
In recent years, one of the most popular applications of machine learning methods are natural language processing tasks, e.g. learning hierarchical representations of words through unsupervised learning. The aim of embedding methods is to position the words in the ambient space such that the distance reflects their semantic and functional similarity. The numerical experiments in~\cite{chamberlain2017neural, nickel2017poincare} suggest that the hyperbolic space $\mathbb H^d$ for $d \in \N$ is particularly well-suited as an ambient space. As a heuristic argument for the empirical findings we note that the volume of a ball in the hyperbolic space increases exponentially with respect its radius. This makes $\mathbb H^d$ a natural choice for embedding tree-like structures~\cite{sala2018representation}, which appear naturally e.g. in many real-world information networks \cite{adcock2013tree}.

Using the hyperboloid model of the hyperbolic space, there exist simple expressions for the Riemannian gradient and exponential map defined in the ambient Minkowski space, see~\cite{wilson2018gradient}.
Since $\mathbb H^d$ has constant sectional curvature, the Riemannian curvature tensor is bounded, see e.g. \cite[Lemma~3.4]{do1992riemannian}. Thus, Theorem~\ref{thm:main1} can be applied to the hyperbolic space.

The exponential growth of the volume of a ball w.r.t. its radius makes it impossible to find an isometric embedding with uniform tubular neighborhood in the sense of Definition~\ref{def:BG}. Hence, in order to apply the second order approximation result, Theorem~\ref{theo:main2}, it remains to verify Definition~\ref{def:reg} for the hyperbolic space, which is left for future research.

\subsection{Statistical manifolds} \label{sec:varinf}
A key task in generative AI is the inference of a probability distribution in a parametrized family $(\nu_\theta)_{\theta \in \Theta}$ of probability measures on $\R^{d_\mathrm{data}}$ that is a good approximation to a given distribution $\nu_{\mathrm{data}}$. Assuming that $\nu_\theta$ is absolutely continuous with respect to $\nu_{\mathrm{data}}$ for all $\theta \in \Theta$, this can be performed by minimizing the Kullback-Leibler divergence
\begin{align*}
	D_{\operatorname{KL}}(\nu_\theta ||\nu_{\mathrm{data}} ) := \int  \log \Bigl( \frac{d \nu_{\theta}}{d \nu_{\mathrm{data}}}  \Bigr) d \nu_{\theta}.
\end{align*}
In the training of generative adversarial networks, the Kullback-Leibler divergence is often approximated by choosing a set of discriminators $D: \R^{d_\mathrm{data}} \to [0,1]$ that try to distinguish between samples from the true data distribution $\nu_{\mathrm{data}}$ and the distribution $\nu_\theta$, see \cite{goodfellow2014generative}. This leads to the objective function
\begin{align*}
	f(\theta) :=  \max_D \int \log(D(x)) \, \nu_\theta (dx).
\end{align*}

Let $\Theta \subset \R^d$ be an open set and assume that, for all $\theta \in \Theta$, $\nu_\theta$ is a probability measure on $\R^{d_\mathrm{data}}$ with density function $p_\theta$. Then, the so-called Fisher information metric is given by
\begin{align*}
	g_\theta(e_i,e_j) := - \int \frac{d^2 \log(p_\theta(y))}{dx_i dx_j} p_\theta(y) dy,
\end{align*}
for $\theta \in \Theta$,
where $e_1, \dots, e_d$ denotes the standard basis of the tangent space $T_\theta \Theta \simeq \R^d$. If $F_\theta := (g_\theta(e_i,e_j))_{i,j =1, \dots, d} $ is positive definite for all $\theta \in \Theta$ this defines a Riemannian metric on $\Theta$. One can write the Riemannian gradient of $f$ as
$
\grad f(\theta) = F_\theta^{-1} \, Df(\theta),
$
where $Df(\theta)$ denotes the Euclidean gradient, i.e. the vector of partial derivatives of $f:\Theta \to \R$ at $\theta$, see e.g.~\cite[Section~3.6]{absil2009optimization}.

For example, consider the set of univariate normal distributions $\Theta= \R \times \R_{>0} \ni \theta= (\mu, \sigma) \mapsto \cN(\mu, \sigma)$. Then, the Fisher information metric is given by the matrix 
\begin{align} \label{eq:346753853}
	F_{(\mu,\sigma)}= \begin{pmatrix}
		\frac{1}{\sigma^2} & 0 \\ 0 & \frac{2}{\sigma^2}
	\end{pmatrix},
\end{align}
see e.g.~\cite{costa2015fisher}. This leads to a Riemannian metric (\ref{eq:346753853}) on the upper half plane $\Theta$ with constant negative curvature equal to $-\frac 12$, see also \cite{costa2015fisher}. Therefore, the geometry of $\Theta$ with Riemannian metric (\ref{eq:346753853}) is similar to the hyperbolic space considered in Section~\ref{sec:hyperbolic}. In this situation, Theorem~\ref{thm:main1} can be applied, whereas Theorem~\ref{theo:main2} needs verification of Definition~\ref{def:reg}.  Regarding the Fisher information metric and its distance function for multivariate normal distributions and other families of probability distributions see~\cite{pinele2020fisher, miyamoto2023closed}.

\subsection*{Acknowledgements}
We would like to thank Matthias Erbar, Jan Nienhaus and Kevin Poljsak for fruitful discussions during the preparation of this work. BG acknowledges support by the Max Planck Society through the Research Group "Stochastic Analysis in the Sciences (SAiS)".  NR acknowledges the support by the Royal Society research professorship of Prof. Martin Hairer, RP$\backslash$R1$\backslash$191065.  This work was co-funded by the European Union (ERC, FluCo,	grant agreement No. 101088488). Views and opinions expressed are however those of the author(s) only and do not necessarily reflect those of the European Union or of the European Research Council. Neither the European Union nor the granting authority can be held responsible for them.

\appendix

\section{Notation}
\label{sec:geo}
%In this section, we give a brief introduction to Riemannian geometry and geometric results that are used throughout the article. 
We let $M$ be a $d$-dimensional $\mathcal C^\infty$-Riemannian manifold that is connected and complete. We denote by $T_x M$ the tangent space of $M$ at point $x \in M$, by $\langle \cdot, \cdot \rangle_x$ the scalar product on $T_x M$ that is given by the Riemannian metric and by $\| \cdot \|_x$ the respective norm on $T_x M$. If it is clear from the context we often omit the reference point $x$ in the above notions and briefly write $\langle \cdot, \cdot \rangle$ and $\|\cdot\|$, respectively. Recall that the tangent bundle $TM:=\dot \bigcup_{x \in M}T_x M$ is a $2d$-dimensional $\mathcal C^{\infty}$-manifold. For $\alpha \in \N_0:= \N \cup \{0\}$, we denote by $\mathfrak X^\alpha(M)$ the set of $\mathcal C^\alpha$-vector fields on $M$, i.e. the set of $\mathcal C^{\alpha}$-functions $V:M\to TM$ with $V(x)\in T_x M$. For $V \in \mathfrak X^0(M)$ and $g:M\to \R \in \cC^1(M)$ we denote by $Vg:M\to \R$ the function that is given by $Vg(x) = V(x) g$ for all $x \in M$. For $x \in M$, a $\mathcal C^\infty$-manifold $N$ and a differentiable mapping $\varphi: M \to N$ we denote by $D_x \varphi:T_x M \to T_{\varphi(x)} N$ the differential of $\varphi$ at $x$, i.e. the linear mapping given by
$
(D_x \varphi \,  v)g = v(g\circ \varphi)
$
for all $v \in T_x M$ and $g \in \mathcal C^1(N)$.
Furthermore, for $v\in T_x M$ and $W \in \mathfrak X^{1}(M)$ we denote by $\nabla_v W \in T_x M$ the covariant derivative of $W$ along $v$ that is induced by the Levi-Civita connection. For $V \in \mathfrak X^0(M)$ and $W \in \mathfrak X^1(M)$ we define $\nabla_V W \in \mathfrak X^0(M)$ via $\nabla_V W(x)=\nabla_{V(x)} W$. Note that, for brevity, we write $\nabla_{V(x)} W$ instead of $\nabla_{V(x)} W(x)$, whenever the base point is clear from the tangent vector in the first argument of the Levi-Civita connection.

For $\beta \in \N$ and a $\mathcal C^\beta$-path $\gamma:[a,b]\to M$ and $t \in [a,b]$ we denote by $\dot \gamma_t \in T_{\gamma_t}M$ the differential of $\gamma$ at time $t$, i.e. $\dot \gamma_t = (D_t \gamma_{\cdot})(\frac{d}{ds}|_{s=t}) $. Moreover, for $\alpha \in \{0, \dots, \beta\}$ we denote by $\mathfrak X^{\alpha}(\gamma)$ the set of all $\mathcal C^{\alpha}$-vector fields along $\gamma$, i.e. $\mathcal C^{\alpha}$-mappings $V:[a,b] \to TM$ with $V_t \in T_{\gamma_t}M$ for all $t \in [a,b]$. For $\alpha \in \{1,\dots, \beta\}$, let $\frac {\nabla }{\rd t}: \mathfrak X^\alpha(\gamma) \to \mathfrak X^{\alpha-1}(\gamma)$ be the Levi-Civita connection on $M$ along $\gamma$ and $\Par_\gamma: T_{\gamma_a} M \to T_{\gamma_b} M$ be the parallel transport along $\gamma$. We denote by
$
\exp_x: T_x M  \to M
$
the exponential map at $x\in M$. 
%By the Hopf-Rinow theorem $\exp_x$ is defined on the whole tangent space $T_x M$. 

$R$ denotes the curvature of $M$, i.e. the $(3,1)$-tensor field given by
\begin{align*}
	R(U,V)W = \nabla_U \nabla_V W - \nabla_V \nabla_U W - \nabla_{[U,V]}W,
\end{align*}
for $U,V \in \mathfrak X^{1}(M)$ and $W\in \mathfrak X^2(M)$, where $[U,V] \in \mathfrak X^0(M)$ denotes the Lie bracket given by
\begin{align*}
	[U,V] g = U(Vg)-V(Ug) \quad \text{ for all }g \in \cC^2(M).
\end{align*}

\section{Derivatives of higher order} \label{sec:higher}
In this section, we introduce the Riemannian gradient and Hessian, as well as derivatives of higher order for real-valued functions and vector fields on $M$. For more information on the Riemannian gradient and Hessian we refer the reader to Chapter~4 in~\cite{Lee_2012} as well as Chapter~3 and Chapter~5 in~ \cite{absil2009optimization}.

Let $g \in \cC^\alpha(M)$ for an $\alpha \in \N$. The Riemannian gradient $\grad g \in \mathfrak X^{\alpha-1}(M)$ of $g$ is the unique $\mathcal C^{\alpha-1}$-vector field on $M$ that satisfies for all $x \in M$ and $v \in T_x M$ that
\begin{align*}
	\langle \grad g(x),v \rangle = vg.
\end{align*}

If $\alpha \ge 2$ we define the Riemannian Hessian $\Hess g(x):T_x M \to T_x M$ of $g$ at $x$ as the linear mapping given by $(\Hess g(x))(v)=\nabla_v \grad g$ for all $v \in T_x M$. Using that $\nabla$ is a metric connection, 
%we get for all vector fields $V,W \in \mathfrak X^1(M)$ that
%\begin{align} \label{eq:Hessian1}
%	\langle (\Hess g)V,W \rangle = V(Wg)-(\nabla_V W)g.
%\end{align}
%Since the Levi-Civita connection is torsion-free, one has $\langle (\Hess g)V,W \rangle = \langle (\Hess g)W,V \rangle$.
%For a linear mapping $H:T_xM \to T_x M$ we denote by $\|H\|$ the operator norm of $H$ w.r.t. the norm induced by the Riemannian metric, i.e.
%\begin{align*}
%\|H\|:= \sup\{ \|Hv\|: v\in T_xM \text{ with } \|v\|=1 \}.
%\end{align*}
%
%One can conceive the Riemannian Hessian as 
this defines a symmetric $(2,0)$-tensor field $\nabla^2 g: \mathfrak X^1(M) \times \mathfrak X^1(M) \to \mathcal C^0(M)$ via
\begin{align*}
	\nabla^2 g (V,W) := \langle (\Hess g)V,W \rangle = V(Wg)-(\nabla_V W)g.
\end{align*}
We have $\|\Hess g(x)\|= \|\nabla^2 g(x)\|$, where the left-hand side denotes the operator norm of the linear mapping $\Hess g(x): T_x M \to T_x M$ and the right-hand side denotes the operator norm of the bilinear mapping $\nabla g^2(x): T_x M \times T_x M \to \R$.

This representation allows us to generalize the concept of the Riemannian Hessian to derivatives of higher order. 
For $3\le n \le \alpha$, we define an $(n,0)$-tensor field $\nabla^n g: (\mathfrak X^{n-1}(M))^{n} \to \mathcal C^0(M)$ inductively via
\begin{align*}
	\nabla^n g(V_1, \dots, V_n) = &V_1 ((\nabla^{n-1}g) (V_2, \dots, V_n)) \\
	&- \sum_{i=2}^n (\nabla^{n-1}g) (V_2, \dots, V_{i-1}, \nabla_{V_1}V_i, V_{i+1}, \dots, V_n).
\end{align*}
%Note that $\nabla^n g:(\mathfrak X^{n-1}(M))^n \to \mathcal C^0(M)$ is indeed multilinear and satisfies for all $V_1, \dots, V_n \in \mathfrak X^{n-1}(M)$, $\varphi \in \mathcal C^{n-1}(M)$ and $1\le i \le n$ that
%\begin{align*}
%(\nabla^ng(V_1,\dots, V_{i-1},\varphi V_i, V_{i+1}, \dots, V_n))(x) = \varphi(x) \, ( \nabla^n g(V_1, \dots, V_n))(x).
%\end{align*}
%This implies that, for every $x \in M$, $\nabla^n g(V_1, \dots, V_n)(x)$ only depends on $V_1(x), \dots, V_n(x)$.
This canonically defines a multilinear mapping $\nabla^ng(x):(T_x M)^n \to \R$ via
\begin{align*}
	(\nabla^n g(x))(v_1, \dots, v_n)= (\nabla^n g(V_1, \dots, V_n))(x),
\end{align*}
where $V_1, \dots, V_n \in \mathfrak X^{n-1}(M)$ are vector fields with $V_i(x)=v_i$ for all $1\le i \le n$. We denote by $\|\nabla^n g(x)\|$ the respective operator norm.
%For a multilinear mapping $H: (T_xM)^n \to \R$ we denote by $\|H\|$ the norm given by
%\begin{align*}
%\|H\| := \sup\{\|H(v_1, \dots, v_n)\|: v_i \in T_xM \text{ with } \|v_i\|=1 \text{ for all } 1\le i \le n\}.
%\end{align*}

If $M$ is an open subset of a Euclidean space $\R^N$ one has $\nabla_{E_i} E_j=0$ for all $i,j=1, \dots, N$ where $E_i(x)=\frac{d}{dx_i}\einschraenkung_x$. Thus, in that case $\nabla^n g$ is given by $D^n g$ which denotes the tensor given by all Euclidean derivatives of $g$ of order $n$.

Analogously, one can define the derivatives of vector fields $V \in \mathfrak X^\alpha(M)$ for $\alpha \in \N$. For $n \le \alpha$ we define an $(n,1)$-tensor field $\nabla^n V: (\mathfrak X^{n-1}(M))^n \to \mathfrak X^0(M)$ inductively via $(\nabla^1 V)(V_1)= \nabla_{V_1} V$ and 
\begin{align*}
	(\nabla^n V)(V_1, \dots, V_n) = &\nabla_{V_1} ((\nabla^{n-1}V) (V_2, \dots, V_n)) \\
	&- \sum_{i=2}^n (\nabla^{n-1}V) (V_2, \dots, V_{i-1}, \nabla_{V_1}V_i, V_{i+1}, \dots, V_n).
\end{align*}
This induces a multilinear mapping $\nabla^n V(x):(T_x M)^n \to T_x M$ and we denote by $\|\nabla ^n V(x)\|$ its operator norm.
%the norm given by
%\begin{align*}
%\|\nabla^n V(x)\| := \sup\{\|(\nabla^n V(x))(v_1, \dots, v_n)\|: v_i \in T_xM \text{ with } \|v_i\|=1 \text{ for all } 1\le i \le n\}.
%\end{align*}

\begin{definition} \label{def:bounded}
	\begin{enumerate}
		\item[(i)] We denote by $\mathcal C^0_b(M)$ the set of continuous and bounded functions $g: M \to \R$ and associate to $g$ the norm 
		$
		\|g\|_{\cC^0_b(M)}:= \sup_{x \in M} |g(x)|.
		$
		Analogously, for $\alpha \in \N$ we denote by $\cC^\alpha_b(M)$ the set of functions $g \in \mathcal C^\alpha(M)$ that satisfy
		\begin{align*}
			\|g\|_{\cC^\alpha_b(M)}:= \sup_{x \in M} \Bigl( |g(x)| \vee \|\grad g(x)\| \vee \max_{2\le n \le \alpha} \|\nabla^n g(x)\|  \Bigr)<\infty.
		\end{align*}
		\item[(ii)] We denote by $\mathfrak X^0_b(M)$ the set of continuous and bounded vector fields $V \in \mathfrak X^0(M)$ and associate to $V$ the norm
		$
		\|V\|_{\mathfrak X^0_b(M)} := \sup_{x \in M} \|V(x)\|.
		$
		Analogously, for $\alpha \in \N$ we denote by $\mathfrak X^\alpha_b(M)$ the set of vector fields $V \in \mathfrak X^\alpha(M)$ that satisfy
		\begin{align*}
			\|V\|_{\mathfrak X^\alpha_b(M)}:= \sup_{x \in M} \Bigl( |V(x)|  \vee \max_{1\le n \le \alpha} \|\nabla^n V(x)\|  \Bigr)<\infty.
		\end{align*}
	\end{enumerate}
\end{definition}

Next, we show how the consecutive differentiation of $g$ w.r.t. multiple vector fields can be expressed in terms of the multilinear mappings defined above. Let us start with second order derivatives. Let $V_1,V_2 \in \mathfrak X^1(M)$ and $g \in \mathcal C^2(M)$. Then, by definition,
\begin{align*}
	V_1V_2g = \nabla^2 g(V_1,V_2) +\langle (\nabla V_2)(V_1), \grad g \rangle,
\end{align*}
where the right-hand side only depends on $\nabla^2g$, $\grad g, V_1,V_2$ and $\nabla V_2$. Thus, $\|V_1V_2g\|_{\cC^0_b(M)}$ can be bounded by a constant that only depends $\|g\|_{\cC^2_b(M)}$, $\|V_1\|_{\mathfrak X^0_b(M)}$ and $\|V_2\|_{\mathfrak X^1_b(M)}$. Analogously, for $V_1,V_2,V_3 \in \mathfrak X^2(M)$ and $g \in \mathcal C^3(M)$ we get
\begin{align*}
	\begin{split}
		V_1V_2V_3g = &\nabla^3g(V_1,V_2,V_3) + \nabla^2 g(\nabla_{V_1}V_2,V_3)+ \nabla^2 g(V_2, \nabla_{V_1}V_3) + V_1 (\nabla_{V_2} V_3)g \\
		= &\nabla^3g(V_1,V_2,V_3) + \nabla^2 g(\nabla_{V_1}V_2,V_3)+ \nabla^2 g(V_2, \nabla_{V_1}V_3) + \nabla^2 g(V_1,\nabla_{V_2} V_3 ) \\
		& + \langle \nabla_{V_1}\nabla_{V_2} V_3 ,\grad g\rangle,
	\end{split}
\end{align*}
where the right-hand side of the equation above can be bounded by a constant that only depends on $\|g\|_{\cC^3_b(M)}$, $\|V_1\|_{\mathfrak X^0_b(M)}$, $\|V_2\|_{\mathfrak X^1_b(M)}$ and $\|V_3\|_{\mathfrak X^2_b(M)}$.
Lastly, for $V_1,V_2,V_3,V_4 \in \mathfrak X^3(M)$ and $g \in \mathcal C^4(M)$ we get
\begin{align*}
	\begin{split} 
		V_1V_2V_3&V_4 g = \nabla^4 g(V_1, V_2, V_3, V_4) + \nabla^3 g(\nabla_{V_1}V_2, V_3, V_4) + \nabla^3 g(V_2, \nabla_{V_1}V_3,V_4)\\
		&+ \nabla^3 g (V_2, V_3, \nabla_{V_1}V_4) + V_1V_2(\nabla_{V_3}V_4) g + V_1 (\nabla_{V_2} V_3)V_4 g - V_1 (\nabla_{\nabla_{V_2}V_3}V_4) g \\
		&+ V_1V_3 (\nabla_{V_2}V_4)g - V_1 (\nabla_{V_3}\nabla_{V_2}V_4) g,
	\end{split}	
\end{align*}
where, after a straight-forward computation, the right-hand side of the equation above can be bounded by a constant that only depends on $\|g\|_{\cC^4_b(M)}$, $\|V_1\|_{\mathfrak X^0_b(M)}$, $\|V_2\|_{\mathfrak X^1_b(M)}$, $\|V_3\|_{\mathfrak X^2_b(M)}$ and $\|V_4\|_{\mathfrak X^3_b(M)}$. For higher derivatives there exist similar expressions.

%\bibliographystyle{alpha}
%\bibliography{SME_on_manifolds}

\newcommand{\etalchar}[1]{$^{#1}$}

\end{document}